\documentclass{article}

\PassOptionsToPackage{numbers, sort}{natbib}
\usepackage{arxiv}
\usepackage{amsmath}
\usepackage{amsthm}

\usepackage[utf8]{inputenc} % allow utf-8 input
\usepackage[T1]{fontenc}    % use 8-bit T1 fonts
\usepackage{url}            % simple URL typesetting
\usepackage{booktabs}       % professional-quality tables
\usepackage{amsfonts}       % blackboard math symbols
\usepackage{nicefrac}       % compact symbols for 1/2, etc.
\usepackage{xcolor}         % colors
\usepackage{microtype}
\usepackage{graphicx}
\usepackage{subcaption}
\usepackage{booktabs} % for professional tables
\usepackage{xcolor}

\usepackage{hyperref}
\usepackage{pgfplots}
\pgfplotsset{compat = newest}
\usepackage{multirow}

\usepackage{amsmath}
\usepackage{amssymb}
\usepackage{mathrsfs}
\usepackage{mathtools}
\usepackage{amsthm}
\usepackage{algorithm}
\usepackage{algpseudocode}
\usepackage{listings}
\usepackage{makecell}
\usepackage{colortbl}
\usepackage{color}
\usepackage{wrapfig}
\usepackage{cancel}
\usepackage{soul,xcolor}
\usepackage{pifont}
\usepackage{tikz}
\usetikzlibrary{shapes, positioning, arrows.meta, calc, decorations.pathmorphing, quotes}

\usepackage[capitalize,noabbrev]{cleveref}

\theoremstyle{plain}
\newtheorem{theorem}{Theorem}[section]
\newtheorem{proposition}[theorem]{Proposition}
\newtheorem{lemma}[theorem]{Lemma}
\newtheorem{corollary}[theorem]{Corollary}
\theoremstyle{definition}
\newtheorem{definition}[theorem]{Definition}

\theoremstyle{remark}

\usepackage[textsize=tiny]{todonotes}

\newcommand{\alink}[1]{\href{#1}{paper-link}}

\usepackage{enumitem}
\usepackage{ amssymb }

\definecolor{codebg}{rgb}{0.95,0.95,0.95}
\definecolor{codeblue}{rgb}{0.13,0.13,1}
\definecolor{codegreen}{rgb}{0,0.5,0}
\definecolor{codegray}{rgb}{0.5,0.5,0.5}
\definecolor{codered}{rgb}{0.7,0,0}

\definecolor{citecolor}{HTML}{0071BC}
\definecolor{linkcolor}{HTML}{ED1C24}
\hypersetup{colorlinks=true, linkcolor=linkcolor, citecolor=citecolor,urlcolor=black}

\usepackage{amsmath,amsfonts,bm}

\def\eqref#1{equation~\ref{#1}}
\def\1{\bm{1}}

\DeclareMathAlphabet{\mathsfit}{\encodingdefault}{\sfdefault}{m}{sl}
\SetMathAlphabet{\mathsfit}{bold}{\encodingdefault}{\sfdefault}{bx}{n}

\definecolor{citecolor}{HTML}{0071BC}
\definecolor{linkcolor}{HTML}{ED1C24}
\hypersetup{colorlinks=true, linkcolor=linkcolor, citecolor=citecolor,urlcolor=black}

\title{Understanding Grokking Through A Robustness Viewpoint}
\author{%
    \textbf{Zhiquan Tan$^1$
    \quad
    Weiran Huang$^{2}$\thanks{Correspondence to Weiran Huang (weiran.huang@outlook.com).}}\\[0.3cm]
    $^1$ Department of Mathematical Sciences, Tsinghua University\\
    $^2$ Qing Yuan Research Institute, SEIEE, Shanghai Jiao Tong University
}

\begin{document}

\maketitle

%%%%%%%%%%%%%%%%%%%%%%%%%%%%%%
%%% Abstract

\begin{abstract}
Recently, an interesting phenomenon called grokking has gained much attention, where generalization occurs long after the models have initially overfitted the training data. We try to understand this seemingly strange phenomenon through the robustness of the neural network. From a robustness perspective, we show that the popular $l_2$ weight norm (metric) of the neural network is actually a \emph{sufficient} condition for grokking. Based on the previous observations, we propose perturbation-based methods to speed up the generalization process. In addition, we examine the standard training process on the modulo addition dataset and find that it hardly learns other basic group operations before grokking, for example, the commutative law. Interestingly, the speed-up of generalization when using our proposed method can be explained by learning the commutative law, a \emph{necessary} condition when the model groks on the test dataset. We also empirically find that $l_2$ norm
correlates with grokking on the test data not in a timely way, we propose new metrics based on robustness and information theory and find that our new metrics correlate well with the grokking phenomenon and may be used to predict grokking.

\end{abstract}

%%%%%%%%%%%%%%%%%%%%%%%%%%%%%%
%%% Introduction
\section{Introduction}

The generalization of overparameterized neural networks has been a fascinating topic in the machine learning community, challenging classical learning theory intuitions. Recently, researchers have discovered a interesting phenomenon called grokking, initially observed in training on unconventional Algorithmic datasets \citep{power2022grokking}, and later found in traditional tasks such as image classification \citep{liu2022omnigrok}. Grokking refers to the unexpected generalization of learning tasks that occurs long after the models have initially overfitted the training data. This phenomenon has garnered increasing attention due to its resemblance to ``phase transition'' \citep{nanda2023progress}. Since its initial report \citep{power2022grokking}, this phenomenon has garnered numerous explanations from various aspects ~\citep{nanda2023progress, liu2022omnigrok,merrill2023tale,barak2022hidden,davies2023unifying,thilak2022slingshot,gromov2023grokking,notsawo2023predicting,varma2023explaining}. 

One popular explanation for grokking is based on the decay of the network's ($l_2$) weight norm ~\citep{liu2022omnigrok, nanda2023progress}. We plot the accuracies and weight norm curve during training on two standard datasets in Figure \ref{fig:initial_grok}. We can see that when the test accuracy quickly increases, the weight norm quickly drops correspondingly. Then a natural question arises: Is this metric a necessary and sufficient signal for grokking?  From Figure \ref{fig:initial_grok} we can see the decrease of weight norm usually happens before the grokking on the test dataset, making it seemingly a sufficient condition for grokking but not a necessary condition. Indeed, we can prove that this is the case. Roughly speaking, when the network's robustness increases to a certain level, then grokking happens. When the weight norm decreases, the robustness increases, indicating that decay of the weight norm is a sufficient condition for grokking. Motivated by the goal of enhancing robustness, we explore the application of perturbation-based training to accelerate generalization.

For one of the tasks (on Modulo Addition Dataset) in Figure \ref{fig:initial_grok}, comprehending the commutative law of modulo addition is necessary for grokking on this task. One may expect the model to first comprehend this ``easy'' (necessary condition) task and then understand the full task of modulo addition. However, we discover an unusual observation during standard training on the modulo addition dataset \citep{power2022grokking}: the model fails to recognize the commutative law of modulo addition before fully understanding the task, which goes against our intuition. Leveraging this finding, we can explain the effectiveness of our perturbation-based approach by verifying that it successfully learns commutative law on the training dataset before grokking.

On the other hand, recall that in Figure \ref{fig:initial_grok}, we have observed that there is no \emph{simultaneous} change between the $l_2$ norm decay and grokking on the test dataset. In light of this, we propose the utilization of novel metrics derived from robustness theory. Our findings indicate that these new metrics exhibit a strong correlation with the grokking phenomenon. Moreover, we have discovered evidence suggesting that these metrics could be valuable in predicting the presence of grokking.

\begin{figure*}
    \centering
    \begin{subfigure}[b]{0.49\columnwidth}
        \centering
        \includegraphics[width=\linewidth]{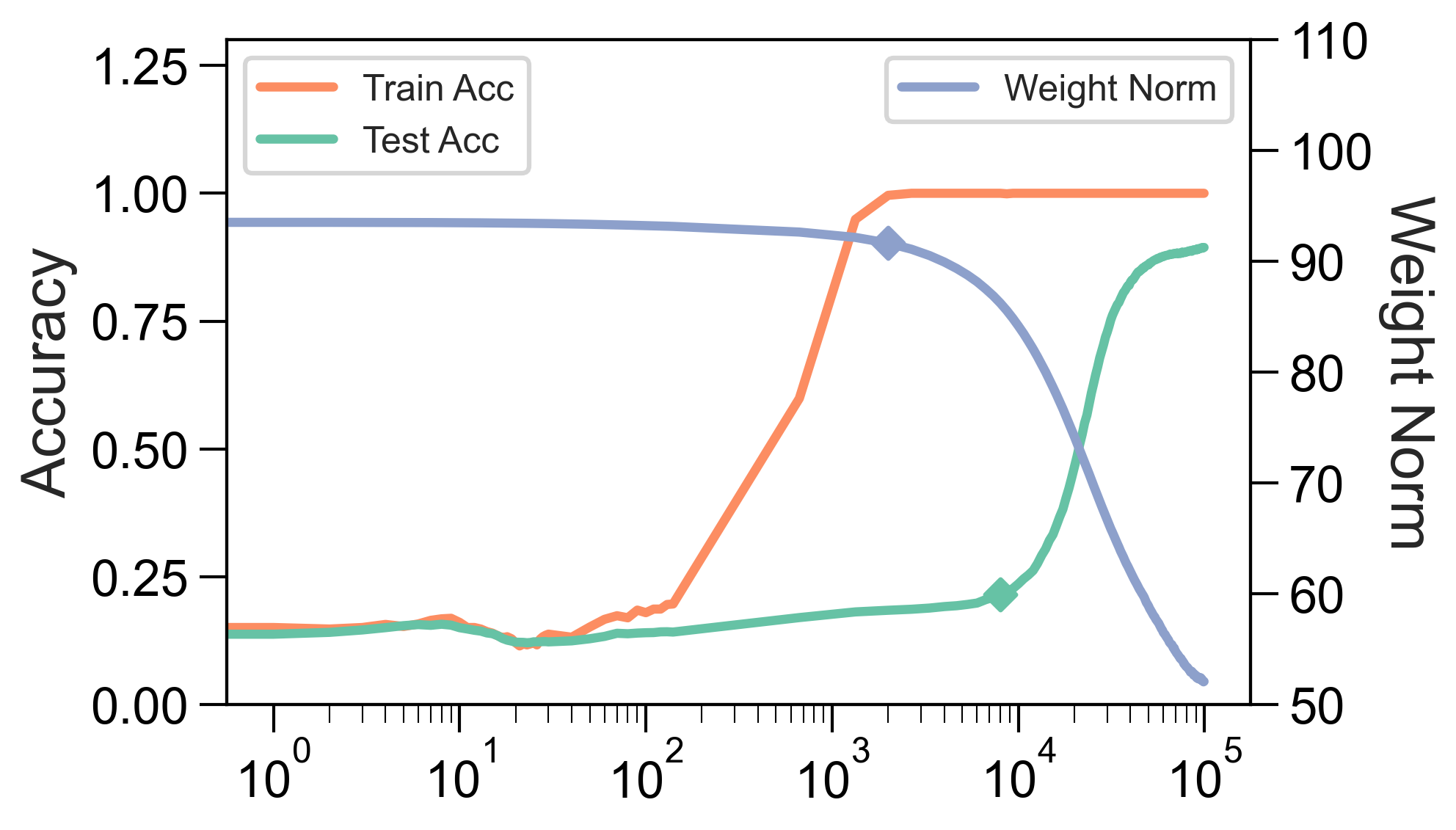}
        \caption{MNIST Dataset}
    \end{subfigure}
    \begin{subfigure}[b]{0.49\columnwidth}
        \centering
        \includegraphics[width=\linewidth]{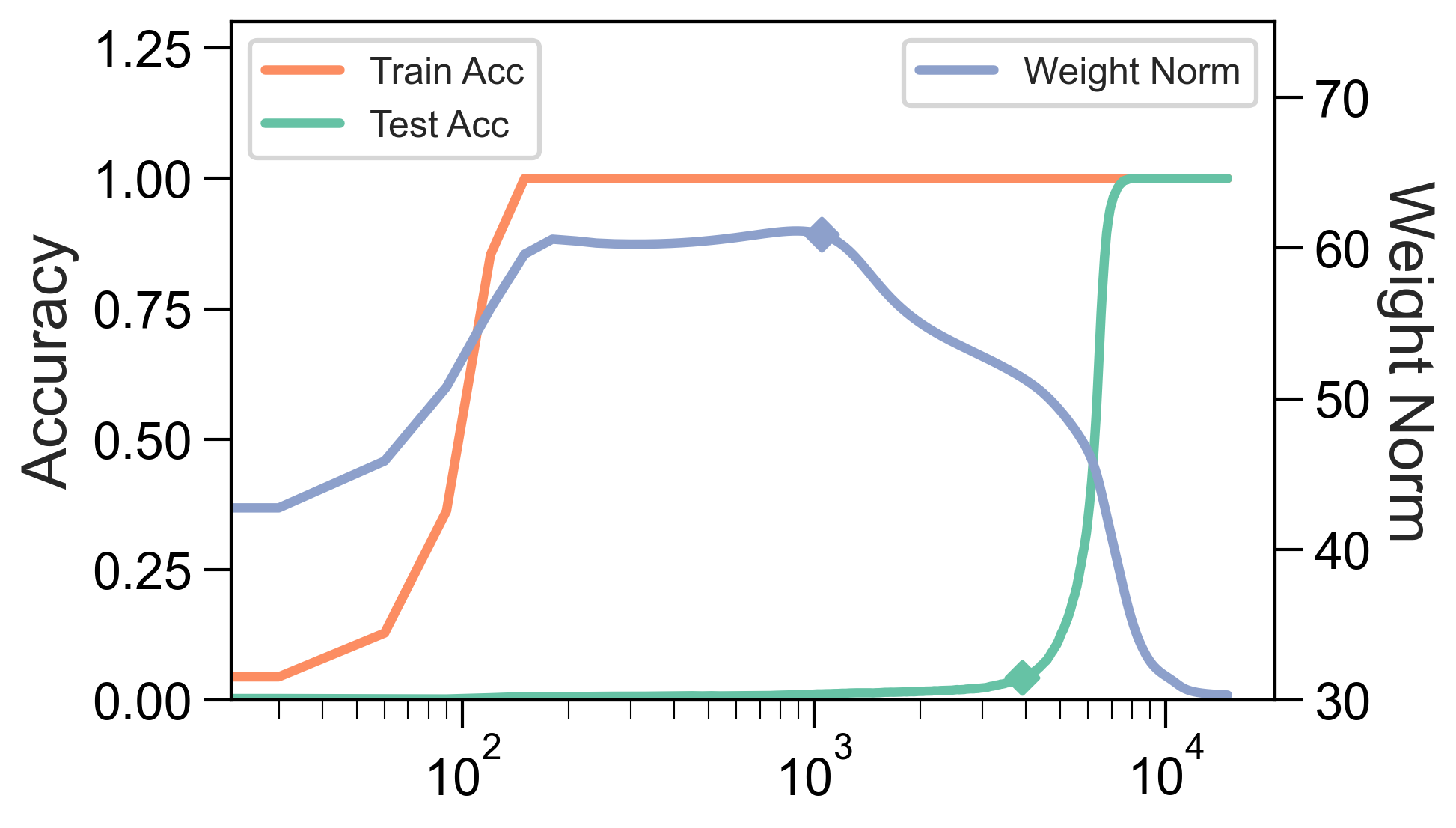}
        \caption{Modulo Addition Dataset}
    \end{subfigure}
    \caption{Typical grokking cases on MNIST and Modulo Addition Dataset: Grokking phenomena refer to sudden and unforeseen enhancements of test accuracy far beyond the point when training accuracy reaches 100\%.}
    \label{fig:initial_grok}
\end{figure*}

Our contributions can be summarized as follows:

\begin{itemize}
\item We theoretically explain why decay of $l_2$ weight norm leads to grokking through the robustness of network. 

\item Motivated by improving robustness, we use perturbation-based training to speed up generalization. 

\item We find a surprising fact that standard training on Modulo Addition Dataset fails to learn commutative law before grokking, which contradicts intuition. We further use this to explain the success of our new perturbation-based strategy.

\item Borrowing ideas from robustness and information theory, we introduce new metrics that correlate better with the grokking process and may be used to predict grokking. 
\end{itemize}

%%%%%%%%%%%%%%%%%%%%%%%%%%%%%%
%%% Method
\section{Preliminary}

\subsection{Problem setting and notations}

In this paper, we mainly consider the setting of supervised classification, the most common case in grokking \citep{liu2022omnigrok, nanda2023progress}. Let $\mathcal{D}_{train} = \{(\mathbf{x}_i, y_i)\}_{i=1}^n$ be the training dataset, and $W$ denote the weight parameters of the neural network, and $f(\mathbf{x}_i, W)$ be the output of neural network on sample $\mathbf{x}_i$. The empirical training loss is defined as $\frac{1}{n} \sum^n_{i=1} \mathcal{\ell}(f(\mathbf{x}_i, W), y_i)$, where $\mathcal{\ell}$ is a loss function such as mean squared error (MSE) loss and cross-entropy loss. In this paper, we consider two canonical grokking settings for experiments: One is the MNIST image classification task introduced by \citep{liu2022omnigrok}; another is the modulo addition dataset \citep{power2022grokking}.

To train our model on the MNIST Dataset, we follow the setup of \citep{liu2022omnigrok}. We adopt a width-200 depth-3 ReLU MLP architecture and employ MSE loss. The optimization is performed using the AdamW optimizer with a learning rate of 0.001, and we utilize a batch size of 200 during the training process.
Our experiments on the Modulo Addition Dataset follow the setup of \citep{nanda2023progress}. We train transformers to perform addition modulo $P$, where the input format is ``$a b =$'' (sometimes we call it ``$a+b=$'', which is more intuitive), with $a$ and $b$ encoded as $P$-dimensional one-hot vectors. The output $c$ is read above the special token ``$=$''. We use $P = 113$ and a one-layer ReLU transformer. The token embeddings have a dimension of $d = 128$, with learned positional embeddings. The model includes 4 attention heads of dimension $\frac{d}{4} = 32$ and an MLP with $k = 512$ hidden units. For evaluation, we use a training dataset comprising $30 \%$ of all possible inputs pairs $(a,b)$, training loss is the standard cross-entropy loss for classification.

Due to the space limitation, all the proofs are deferred to Appendix \ref{proofs}.

\subsection{Matrix information theory}

In this section, we briefly summarize the matrix information-theoretic quantities that we will use  ~\citep{skean2023dime}.

\begin{definition}[$\alpha$-order matrix entropy] Suppose a positive semi-definite matrix $\mathbf{R} \in \mathbb{R}^{n \times n}$ which $\mathbf{R}(i,i)=1$ for every $i=1, \cdots, n$ and $\alpha>0$. The $\alpha$-order (R\'enyi) entropy for matrix $\mathbf{R}$ is defined as follows:
$$
\operatorname{H}_\alpha\left(\mathbf{R}\right)=\frac{1}{1-\alpha} \log \left[\operatorname{tr}\left(\left(\frac{1}{n} \mathbf{R} \right)^\alpha\right)\right],
$$
where $\mathbf{R}^{\alpha}$ is the matrix power.

The case of $\alpha=1$ recovers the von Neumann (matrix) entropy, i.e., 
$$
\operatorname{H}_1\left(\mathbf{R}\right)=-\operatorname{tr}\left(\frac{1}{n} \mathbf{R} \log \frac{1}{n} \mathbf{R} \right).
$$
\end{definition}

In this paper, if not stated otherwise, we will always use $\alpha=1$. Using the definition of matrix entropy, we can define matrix mutual information as follows.

\begin{definition}[Matrix mutual information] Suppose positive semi-definite matrices $\mathbf{R}_1, \mathbf{R}_2 \in \mathbb{R}^{n \times n}$ which $\mathbf{R}_1(i,i) = \mathbf{R}_2(i,i) = 1$ for every $i=1, \cdots, n$. $\alpha$ is a positive real number. The $\alpha$-order matrix mutual information for matrix $\mathbf{R}_1$ and $\mathbf{R}_2$ is defined as follows:
$$
\operatorname{I}_{\alpha}(\mathbf{R}_1; \mathbf{R}_2) = \operatorname{H}_{\alpha}(\mathbf{R}_1) + \operatorname{H}_{\alpha}(\mathbf{R}_2) - \operatorname{H}_{\alpha}(\mathbf{R}_1 \odot \mathbf{R}_2).
$$
    
\end{definition}

\section{Grokking: a robustness perspective}

\subsection{Understanding the role of $l_2$ weight norm decay in Grokking}

From Figure \ref{fig:initial_grok}, we can see that when the network starts to grok, there is usually a sharp decrease in the network's $l_2$ weight norm. However, the decrease in the network $l_2$ weight norm usually starts \emph{before} the test accuracy quickly increases, making this metric a seemingly sufficient condition for grokking. In the following, we will show that this is indeed the case from a theoretical viewpoint.

For the simplicity of theoretical analysis, we consider the MSE loss in this section, e.g., the empirical risk can be written as $\mathcal{L}(W):=\frac{1}{2 n} \sum_{i=1}^n\left \|f\left (\mathbf{x}_i, W\right)-\text{onehot}(y_i)\right\|^2_2$.

Let $W^*=\left(W_1^*, W_2^*\right)$ be an interpolation solution, i.e., $\mathcal{L}\left(W^*\right)=0$, namely, $f\left(\mathbf{x}_i, W^*\right)= \text{onehot}(y_i)$ on the entire training dataset.
Define sharpness of the loss function at $W^*$ as the sum of the eigenvalues of the Hessian $\nabla^2 \mathcal{L}\left(W^*\right)$, i.e.,  $S\left(W^*\right):=\operatorname{Tr}\left(\nabla^2 \mathcal{L}\left(W^*\right)\right)$.

\begin{figure*}
\centering
\begin{subfigure}[b]{0.24\columnwidth}
\includegraphics[width=\linewidth]{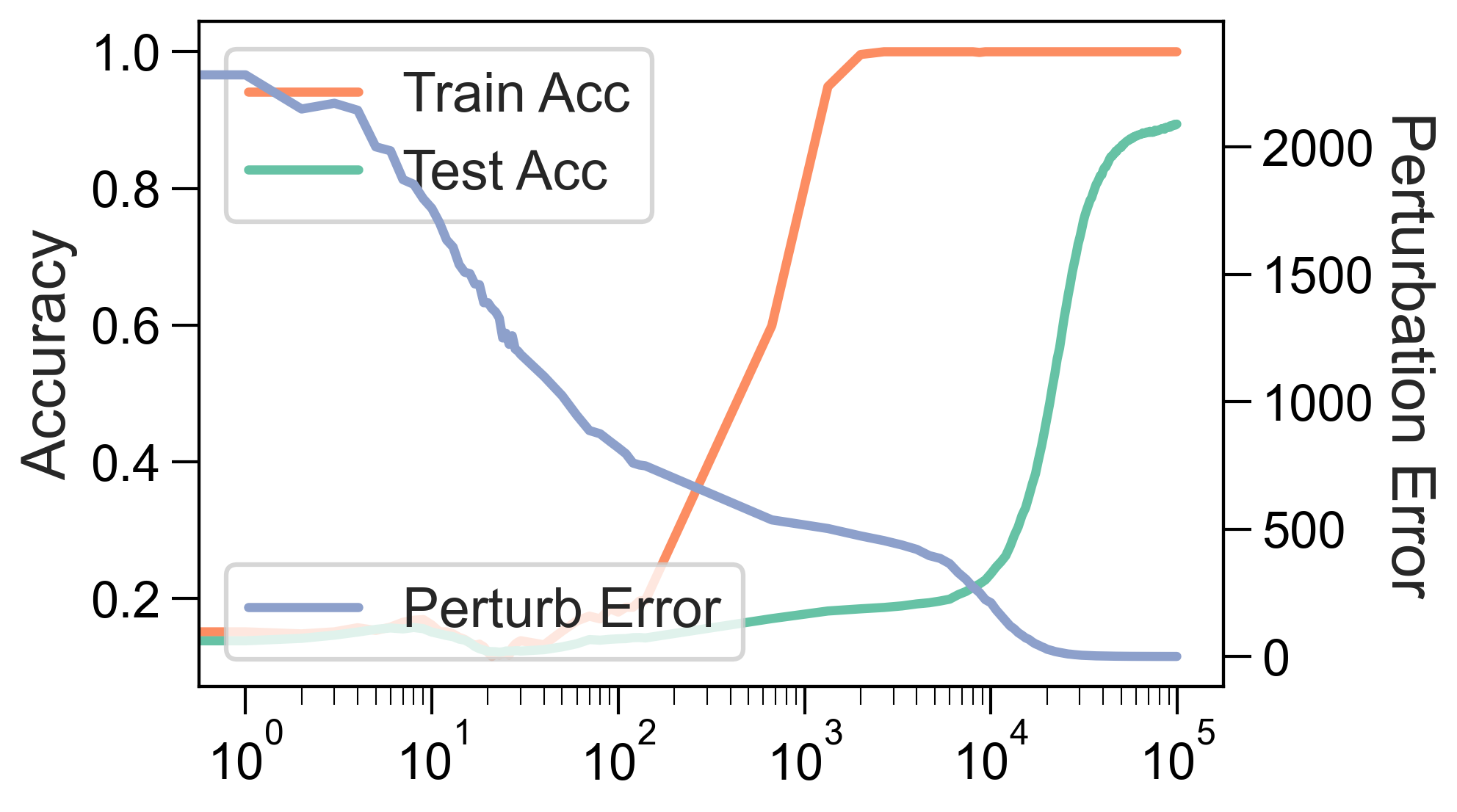}
\caption{MNIST perturb error}
\end{subfigure}
\begin{subfigure}[b]{0.24\columnwidth}
\includegraphics[width=\linewidth]{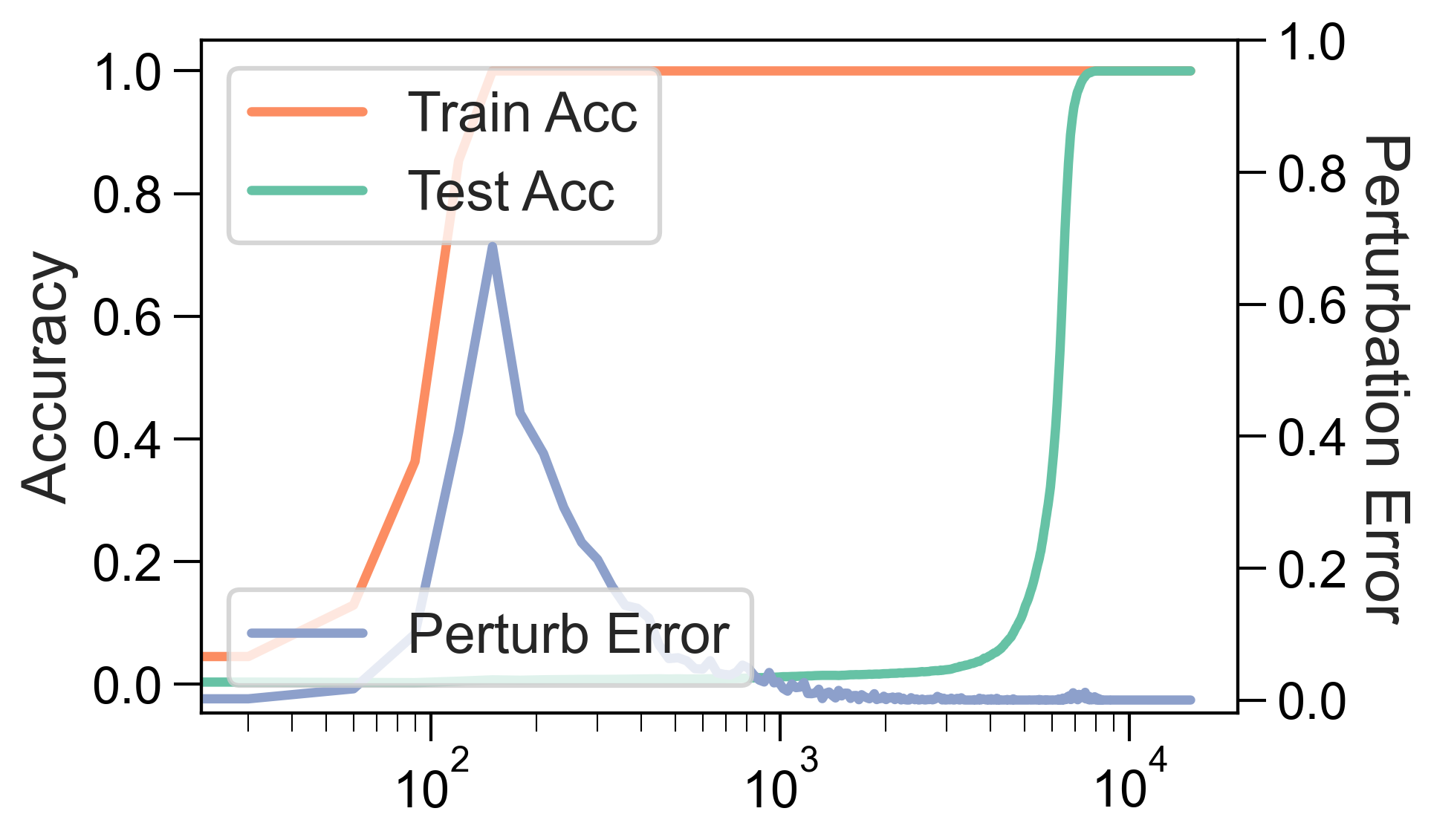}
\caption{Algorithmic perturb error}
\end{subfigure}
\begin{subfigure}[b]{0.24\columnwidth}
\includegraphics[width=\linewidth]{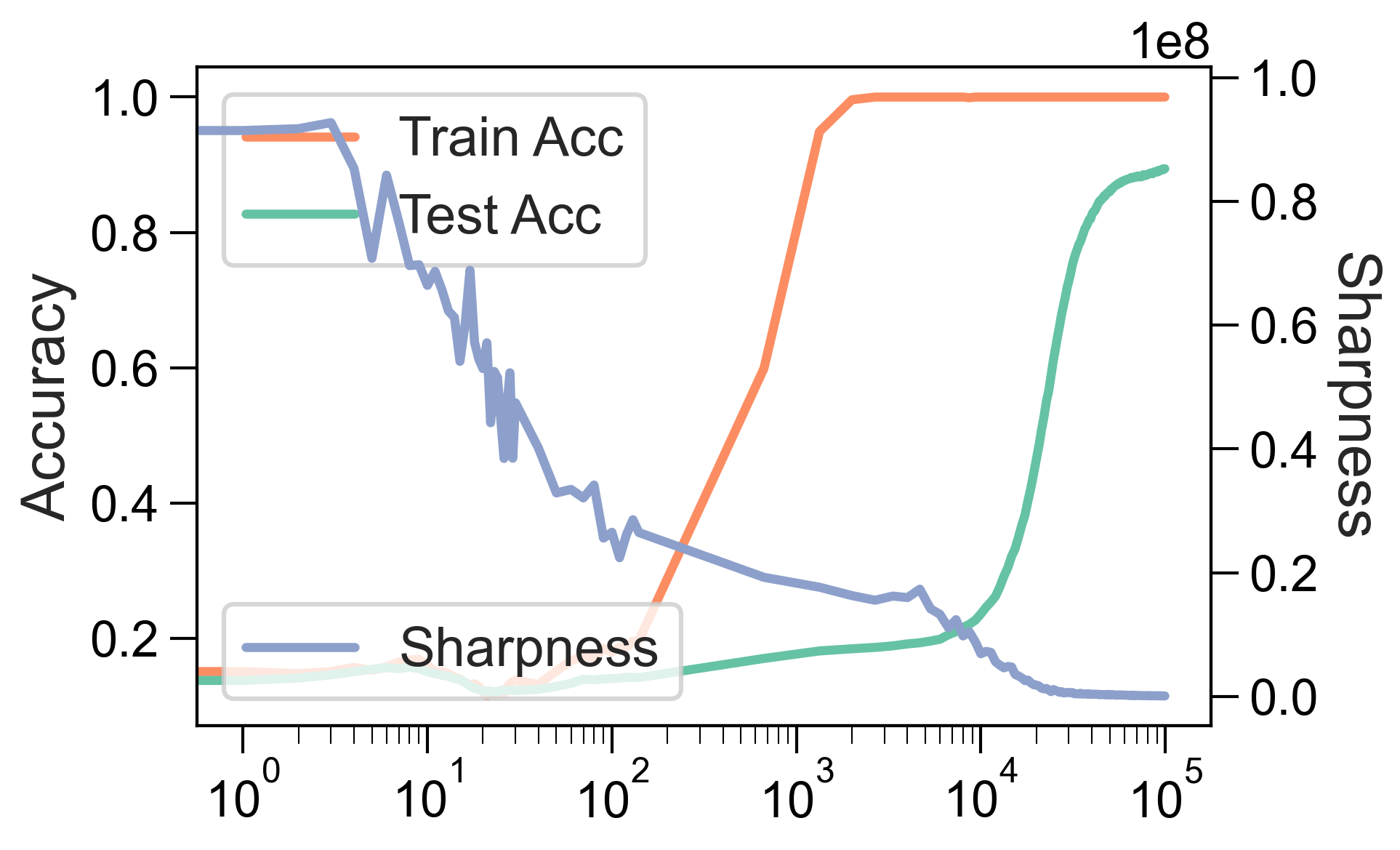}
\caption{MNIST sharpness}
\end{subfigure}
\begin{subfigure}[b]{0.24\columnwidth}
\includegraphics[height=0.58\linewidth, width=\linewidth]{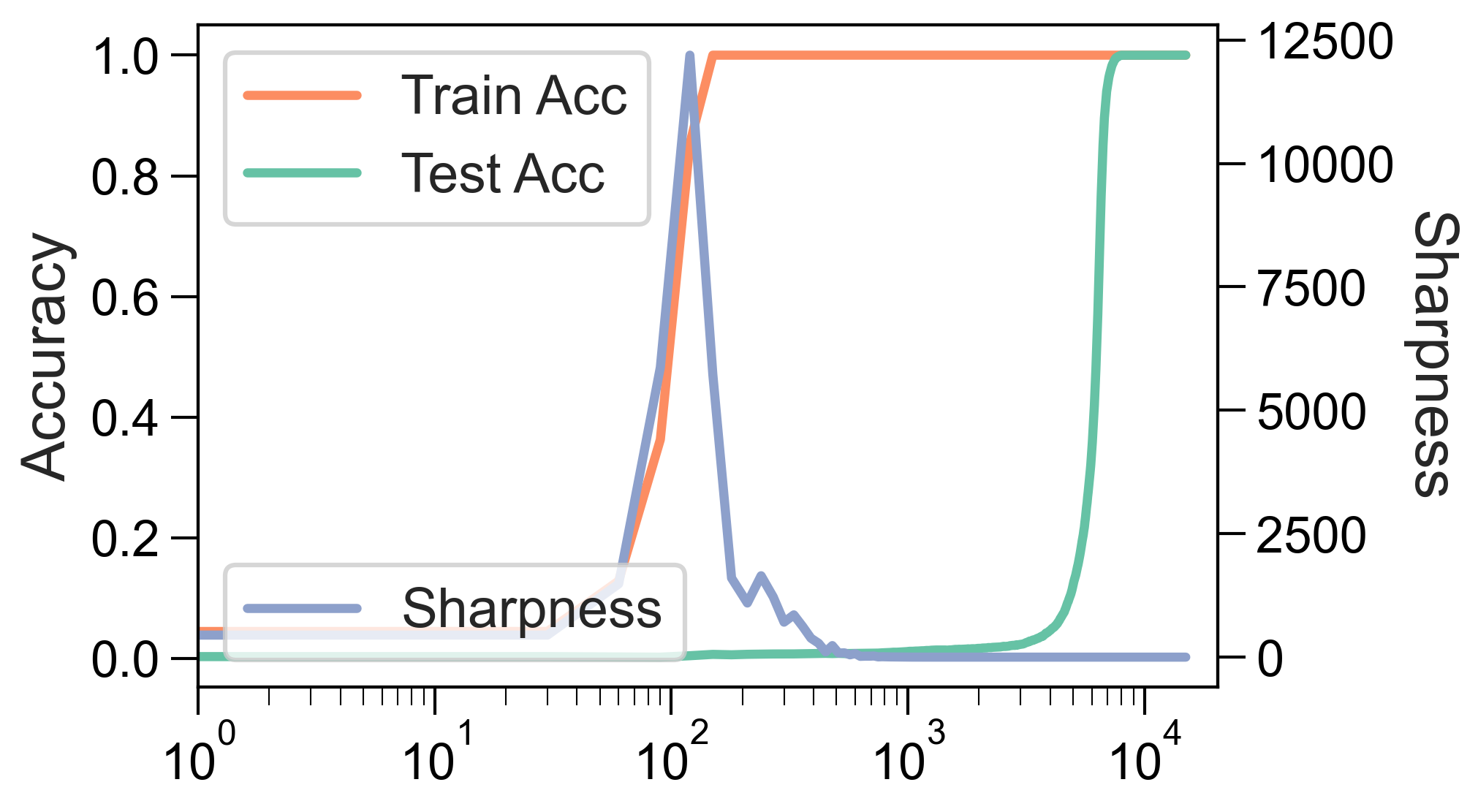}
\caption{Algorithmic sharpness}
\end{subfigure}

\caption{Robustness and sharpness on MNIST and Modulo Addition training dataset, we use Algorithmic to indicate Modulo Addition dataset to save the space for image caption. (Set the deviation of Gaussian perturbation as $\sigma= 0.04$). We plot the perturbation error as a robustness indicator in (a) and (b), a \textbf{smaller} perturbation error means more robust.}
\label{fig:hessian grok}
\end{figure*}

We will first present a lemma which is mainly adapted from \citep{ma2021linear}. This lemma describes the relationship of weight norm, sharpness, and the robustness of a neural network.
\begin{lemma} \label{robust lemma}

Suppose $W^*$ is an interpolation solution, then the following inequality holds:

\begin{equation*}
    \frac{1}{n} \sum_{i=1}^n\left\|\nabla_{\mathbf{x}} f\left(\mathbf{x}_i, W^*\right)\right\|^2_F \leq \frac{\left\|W^*\right\|_F^2}{\min _i\left\|\mathbf{x}_i\right\|^2_2} S\left(W^*\right) .
\end{equation*}   
\end{lemma}

To see how the robustness of network behaves during training, we will try to directly evaluate the perturbation error (a \textbf{smaller} perturbation error means more robust): $\sum^n_{i=1} \|f(\mathbf{x}_i + \Delta_i, W) - f(\mathbf{x}_i, W)  \|^2_F$, where $\Delta_i \sim \mathcal{N}(0, \sigma^2 \mathbf{I})$ is Gaussian noise. We plot the perturbation error in Figure \ref{fig:hessian grok} (a) and (b) and find it decreases after train accuracy reaches $100 \%$. As lemma \ref{robust lemma} also involves the sharpness. We plot the sharpness in Figure \ref{fig:hessian grok} (c) and (d) and find it continuously decreases during training. Then as the weight norm also decreases, lemma \ref{robust lemma} shows that the perturbation error will decrease which fits empirical results.

Then we use lemma \ref{robust lemma} to understand grokking, \emph{informally} we show: Under some mild assumptions, we can successfully classify all the test samples in a certain ``diameter'' of training samples, lemma \ref{robust lemma} is used to determine the expression of ``diameter''. Formally, we have theorem \ref{l2 decay}. This theorem can be seen as showing that decreasing of $l_2$ weight norm is a sufficient condition to explain grokking. 

\begin{theorem} \label{l2 decay}
Suppose $W^*$ is a interpolation solution and the gradient of $f(\mathbf{x}, W^*)$ is $L$-Lipschitz about $\mathbf{x}$. Suppose at least $\delta$-fraction of test data has a train dataset neighbour whose distance is at most $\epsilon(W^*)$, where $\epsilon(W^*) = \min\{1, \frac{1}{2(\sqrt{\frac{n}{\min _i\left\|\mathbf{x}_i\right\|^2_2}\left\|W^*\right\|_F^2 S\left(W^*\right)}+L)}  \}$. Then the test accuracy will be at least $\delta$.  
\end{theorem}

% \textbf{Remark:} The decreasing of sharpness may also explain the weight normalization strategy in \citep{liu2022omnigrok} where the weight is kept constant at training.

From theorem \ref{l2 decay}, it is clear that when $l_2$ norm decays, the distance threshold $\epsilon(W^*)$ increases thus the ratio of test samples who has a train dataset neighbour of distance at most $\epsilon(W^*)$ increases. Then it is clear that the test accuracy increases, which means that $l_2$ weight decay is indeed a sufficient condition for the generalization on test dataset.

The name grokking consists of two things. One is that it generalizes on test data eventually, which we have discussed previously. Another is that when it generalizes, the accuracy increases sharply, mimicking a sort of ``phase transition'' phenomenon. We will further analyze the latter in detail in the following.

Define the distance function to the training dataset as follows:
\begin{equation*}
d(\mathbf{x}, \mathcal{D}_{train} ) = \min_{\mathbf{y} \in \mathcal{D}_{train}} \|\mathbf{x} - \mathbf{y}  \|_2.
\end{equation*}

Motivated by theorem \ref{l2 decay}, we can define the neighbouring probability of test data distribution $P_{test}$ as:
$P(r) = \mathbb{E}_{\mathbf{X} \sim P_{test}} \mathbb{I}(d(\mathbf{X}, \mathcal{D}_{train} ) \leq r) = \Pr(d(\mathbf{X}, \mathcal{D}_{train} ) \leq r).$

Then theorem \ref{l2 decay} shows that the test accuracy will be at least $P(\epsilon(W^*))$. Therefore, we have the following corollary.

\begin{corollary} \label{phase transition}
Under the same assumption of theorem \ref{l2 decay}. Suppose $L \geq \frac{1}{2}$, $\| \mathbf{x}_i \|_2=1$ and $\left\|W^*\right\|_F^2 S\left(W^*\right) = \frac{\max^2 \{a-b \log_{10} (\text{train-steps} ), 0 \} }{4n}$. The the test accuracy will be at least $P(\frac{1}{2L+\max \{a-b \log_{10} (\text{train-steps} ), 0 \}})$.
\end{corollary}

Assume $Y \sim \mathcal{N}(0, \mu)$ is a normal distribution of variance $\mu$, as the distance function is always non-negative. A very natural form of $P(r)$ will be $P(r) = \Pr(-r \leq Y \leq r)$. We will see that this assumption will make us simulate the ``phase transition'' very closely.

Take $L=\frac{1}{2}$, $\mu = \frac{1}{100}$, $a=1925$ and $b=500$. We take the total training steps as $15000$, which is the same as the algorithmic dataset. We plot the predicted accuracy given by corollary \ref{phase transition} in Figure \ref{fig: theoretical grok}. It is interesting to see that the predicted accuracy closely matches that of the real accuracy. This means that our theory provides an explanation for the ``phase transition'' phenomenon in grokking.

\begin{figure}[t] 
\centering 
\includegraphics[width=0.5\columnwidth]{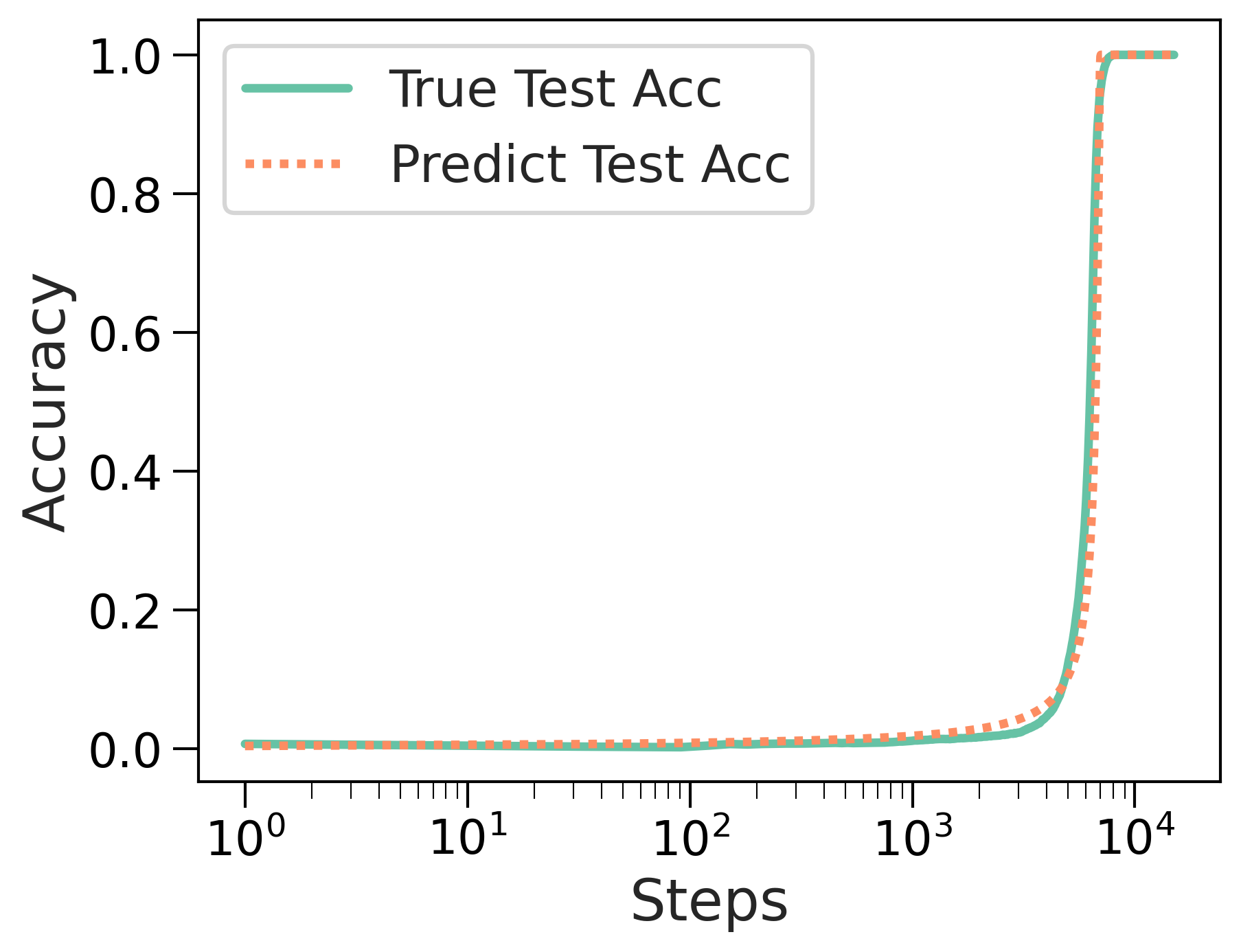}
\caption{The predicted accuracy matches the real test accuracy.}
\label{fig: theoretical grok}
\end{figure}

We will then analyze another type of regularization techniques, for example, $l_1$ weight decay.

\begin{corollary} \label{l1 decay}
Under the same assumption of theorem \ref{l2 decay}.  Suppose at least $\delta$-fraction of test data has a train dataset neighbour whose distance is at most $\hat{\epsilon}(W^*)$, where $\hat{\epsilon}(W^*) = \min\{1, \frac{1}{2(\sqrt{\frac{n}{\min _i\left\|\mathbf{x}_i\right\|^2_2}\left\|W^*\right\|_1^2 S\left(W^*\right)}+L)}  \}$. Then the test accuracy will be at least $\delta$.     
\end{corollary}

\begin{proof}
Denote $\operatorname{vec}$ as the operation of reshaping the tensor to a column vector. Then define $w = \operatorname{vec}(W^*)$, it is clear that $\| W^* \|^2_F = \| w \|^2_2$ and $\| W^* \|_1 = \| w \|_1$. As $\| w \|^2_2 = \sum_i w^2_i = \sum_i | w_i |^2 \leq (\sum_i | w_i |)(\sum_i | w_i |) = \| w \|^2_1$, combined with theorem \ref{l2 decay} the conclusion follows.
\end{proof}

The corollary \ref{l1 decay} can show that $l_1$ norm decay will also result in grokking. \citet{vzunkovivc2022grokking} discuss that $l_1$ is better than $l_2$ decay in $1D$ exponential model, this can be understood by the proof of corollary \ref{l1 decay} as the decay of $l_1$ norm imposes more constraint than $l_2$ norm.

\subsection{A direct degrokking strategy based on robustness} \label{perturb degrok}

Our previous findings suggest that the robustness of a neural network plays a significant role in the process of grokking. In light of this, we need an approach to enhance the robustness of neural networks directly and thus speed up the generalization process (which we termed ``degrokking'').

To achieve this, we have designed a method that introduces controlled perturbations to the input data during the training process. By doing so, we aim to induce the neural network to increase its resilience to variations and uncertainties in the input data, thus increasing the robustness of the network. This approach allows us to maintain the integrity of the initial training while making minimum modifications. We made minimum changes to the initial training by only adding $\Delta \sim  \mathcal{N}(0, \sigma^2 \mathbf{I})$ to the input. As grokking has an extremely unequal speed of convergence on training and testing dataset compared to non-grokking standard training cases, adding a constant strength perturbation to the training input may result in over-perturbation. To better consider the training process, we adaptively update $\sigma = \max(\lambda_1(1-\text{train acc}), \lambda_2)$, this makes the perturbation strength transits much more smoothly.  Formally, the training objective function is as follows:
\begin{equation*}
\frac{1}{n} \sum^n_{i=1} \mathcal{\ell}(f(\mathbf{x}_i +\Delta_i, W), y_i),  
\end{equation*}
where $\mathbf{x}_i$ is the image in MNIST and the input token embedding in the Modulo Addition Dataset.

We plot the curves in Figure \ref{fig:perturb degrok} (a) and (b), it is clear that this perturbation-based strategy speeds up the generalization. On MNIST, $\lambda_1=0.06$ and $\lambda_2=0.03$. On Modulo Addition Dataset, $\lambda_1=0.5$ and $\lambda_2=0.4$. We attribute the staircase-like curve on the algorithmic dataset to ``first-memorize, then generalize'' and we left the detailed study as future work. Compared to \citep{liu2022omnigrok}, our strategy can be seen as a traditional training technique. It may also explain why grokking is not common in usual settings, where data augmentations can be seen as adding a sort of perturbation to the input.

\begin{figure}[htb]
\centering
\begin{subfigure}[b]{0.49\columnwidth}
\includegraphics[width=\linewidth]{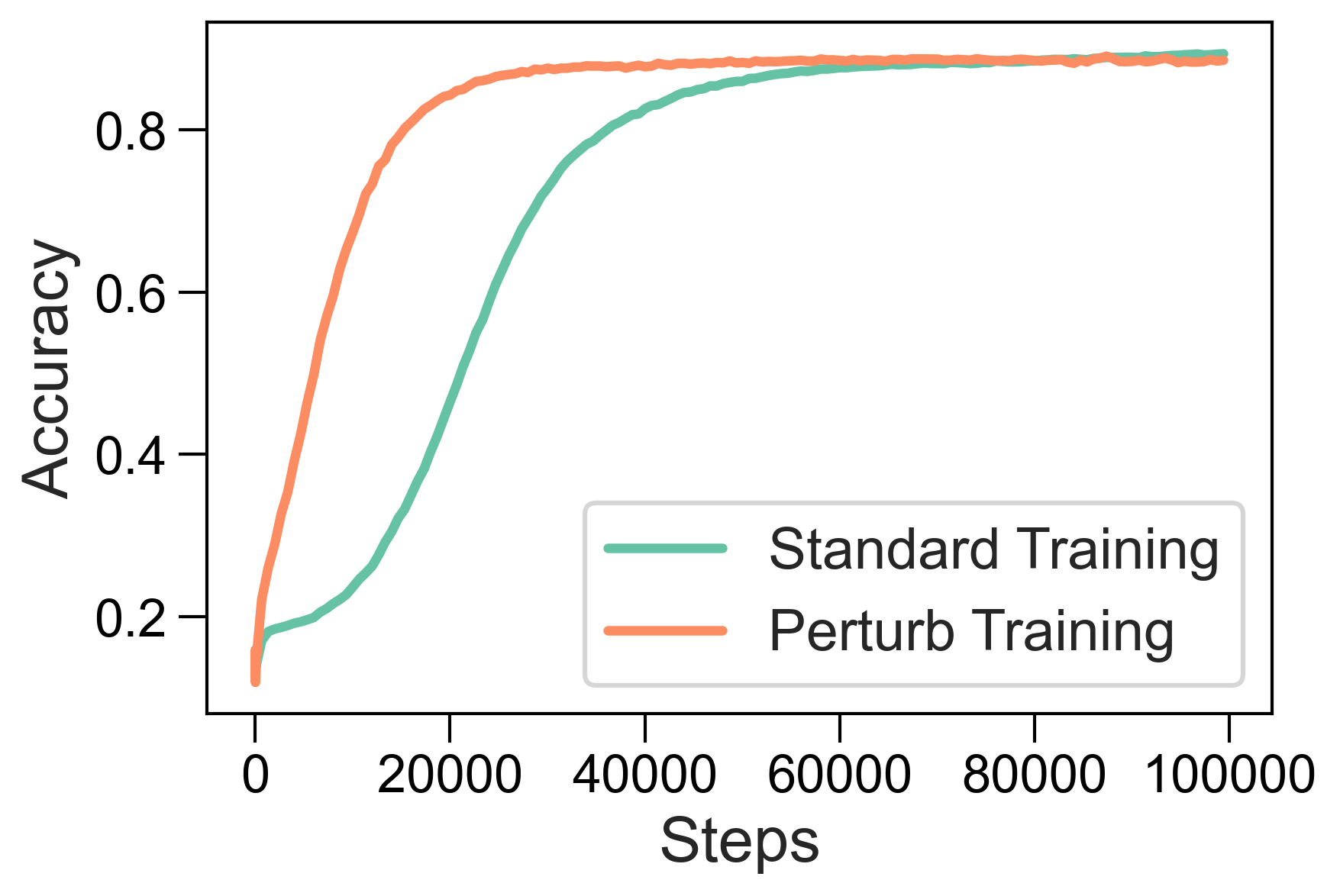}
\caption{MNIST Dataset}
\end{subfigure}
\begin{subfigure}[b]{0.49\columnwidth}
\includegraphics[width=\linewidth]{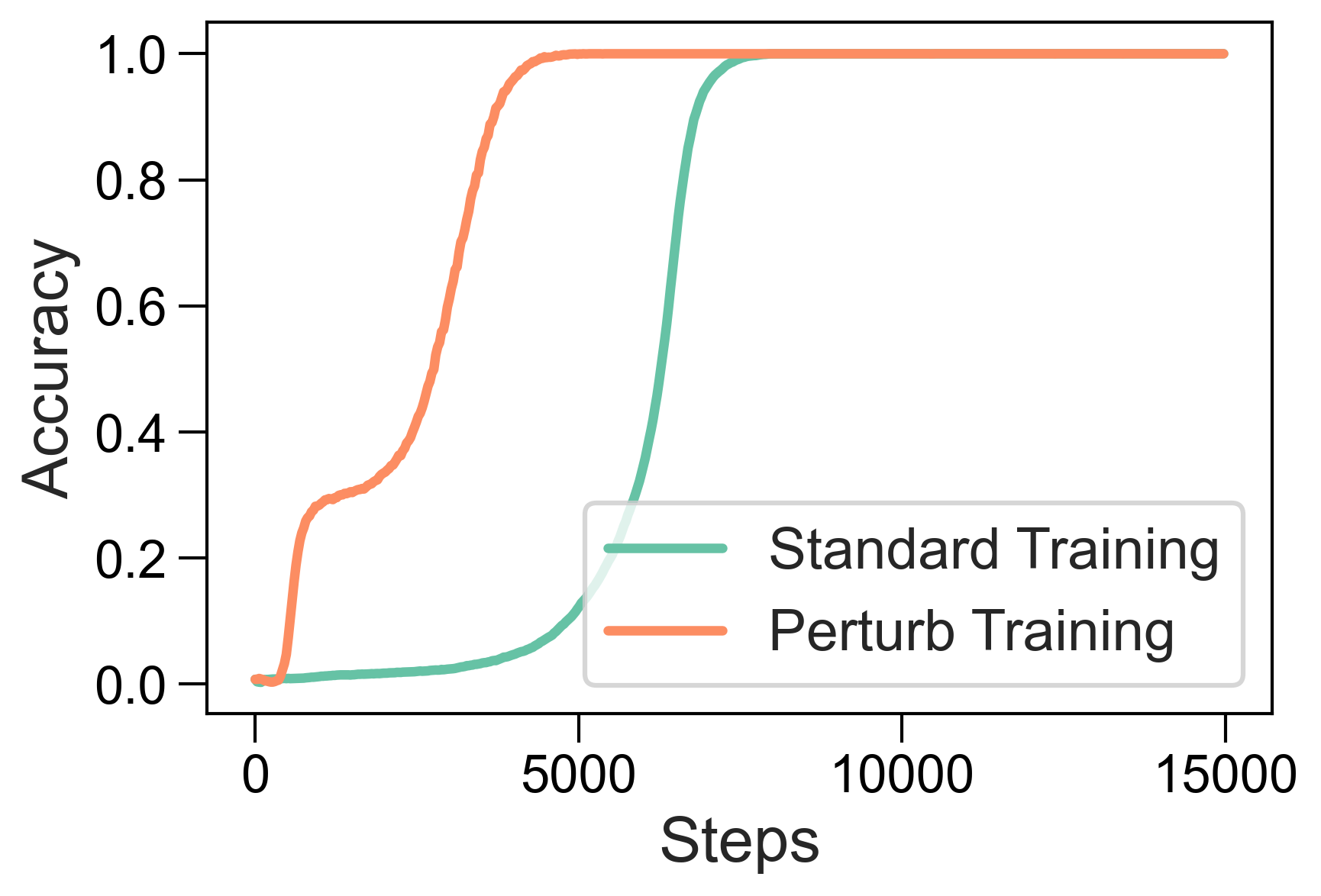}
\caption{Modulo Addition Dataset}
\end{subfigure}
\caption{Our perturbed training strategy speeds up generalization (``degrokking'').}
\label{fig:perturb degrok}
\end{figure}

\section{A closer look at grokking on Modulo Addition Dataset}

\subsection{Necessary condition from group theory}

As the task is modulo adding, from group theory, it shall obey the commutative law once it actually learns the general addition rule.

Formally, the commutative law in an abelian group is
\begin{equation*}
    a + b = b + a,
\end{equation*}
where $+$ is the binary group operation considered.

\begin{figure*}
\centering
\begin{subfigure}[b]{0.49\columnwidth}
\includegraphics[width=\linewidth]{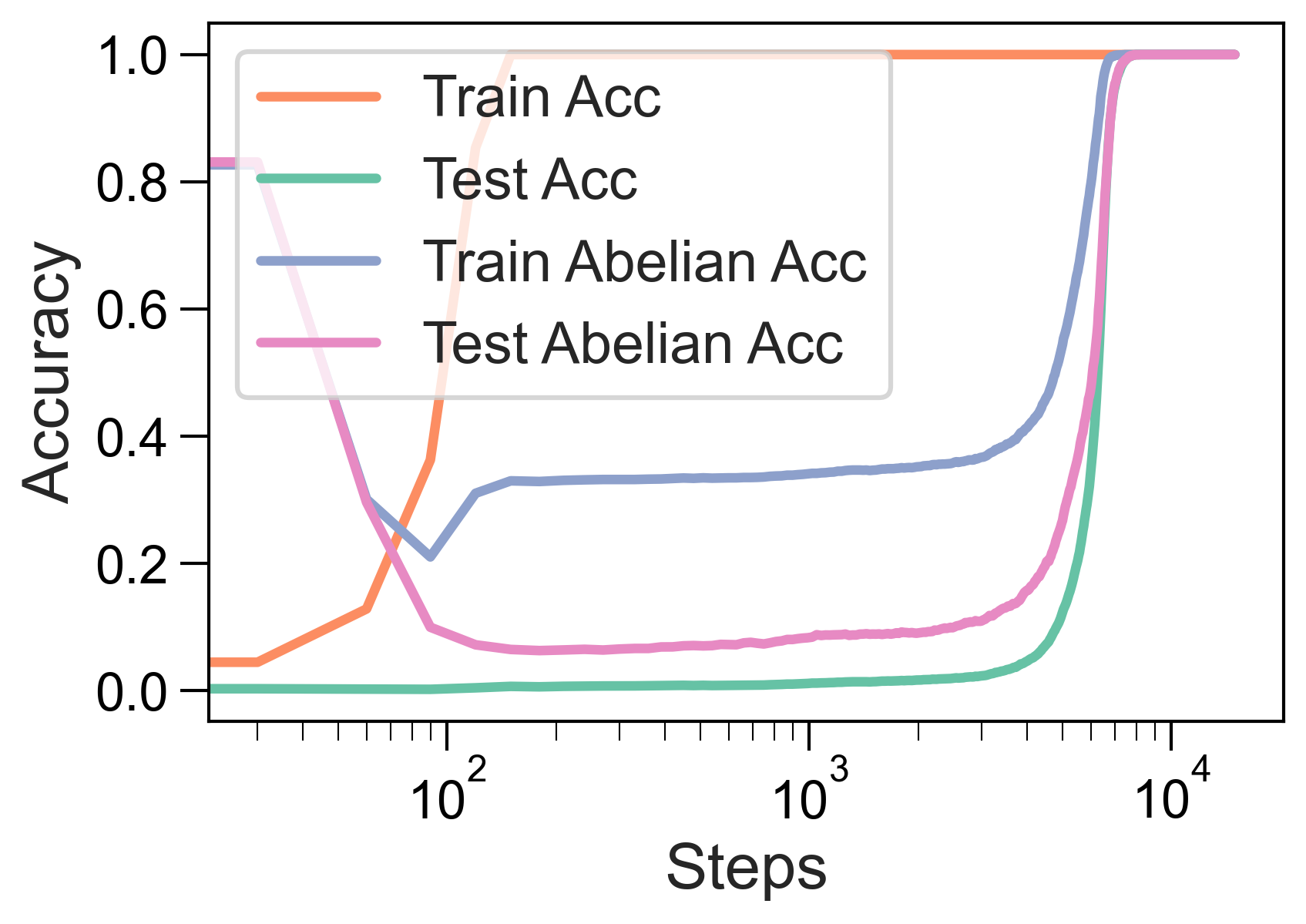}
\caption{Standard training}
\end{subfigure}
\begin{subfigure}[b]{0.49\columnwidth}
\includegraphics[width=\linewidth]{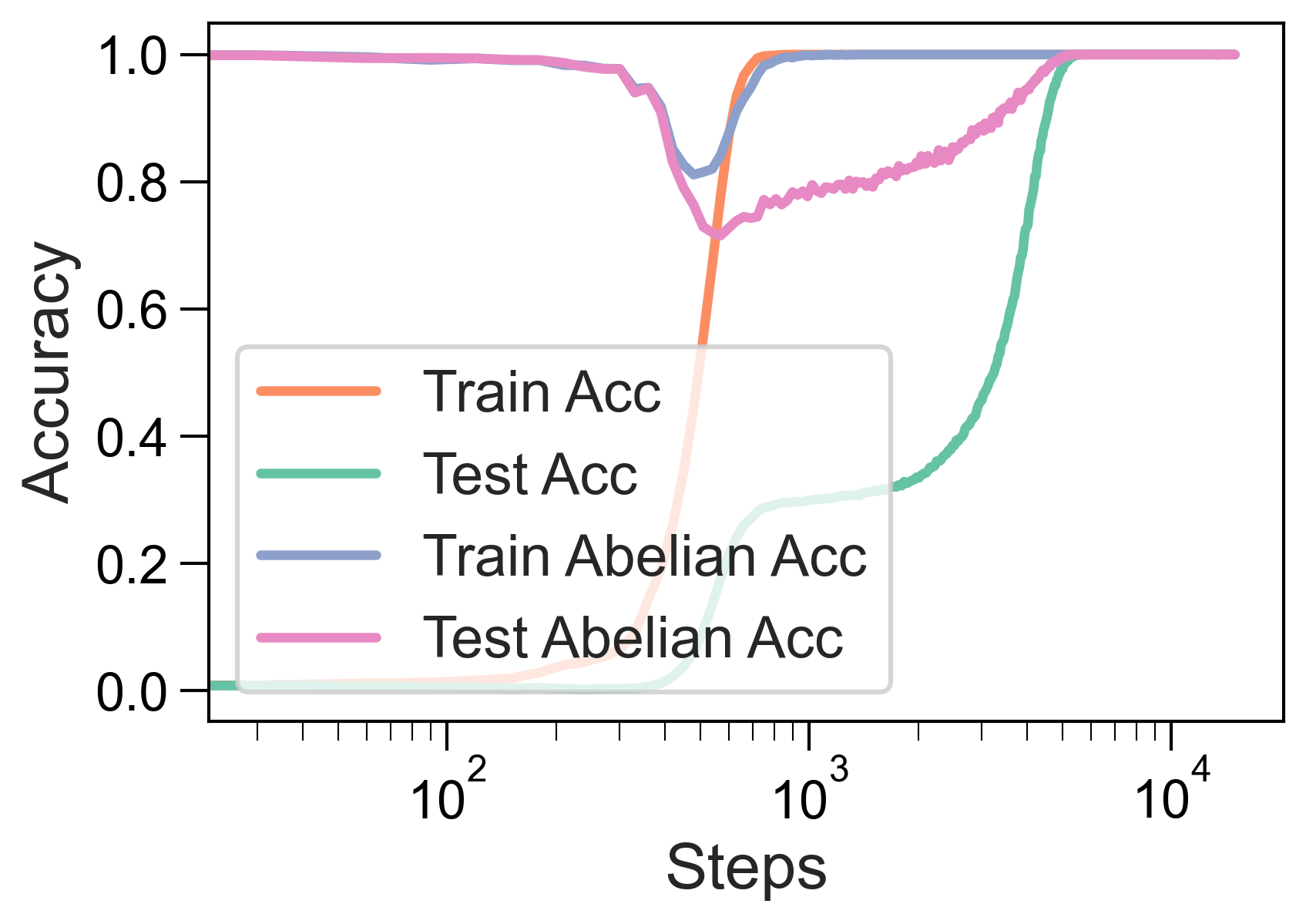}
\caption{Perturb training}
\end{subfigure}
\caption{Abelian test shows that perturb training will help model comprehend commutative rule.}
\label{fig:abelian test}
\end{figure*}

In Figure \ref{fig:abelian test}, we plot the accuracy of when the prediction of $a+b$ is equal to that of $b+a$ on samples $(a, b)$s. Note the prediction is based on the maximal index of the ``$a+b=$'' (``$b+a=$'') logits. We call this ``abelian test''. Surprisingly, from Figure \ref{fig:abelian test}, it is clear that the standard training process does not commutative rule on the training data until it groks. However, the perturbed training strategy comprehends commutative rule on the training data right after the training accuracy reaches $100 \%$. 

In Figure \ref{fig: abelian degrok}, we further add the logits level MSE loss of training samples ``$a+b=$'' and ``$b+a=$'' as regularizer on the initial loss. We call this strategy abelian degrok, in the experiment we set the regularization coefficient of the regularizer to $100$. We find it speeds up training and performs even better than the perturb-based strategy. Therefore, combining the observations in Figures \ref{fig: abelian degrok} and \ref{fig:abelian test}, we think the reason why our perturb-based strategy is effective is that it comprehends the commutative law.

\begin{figure}[t] 
\centering 
\includegraphics[width=0.5\columnwidth]{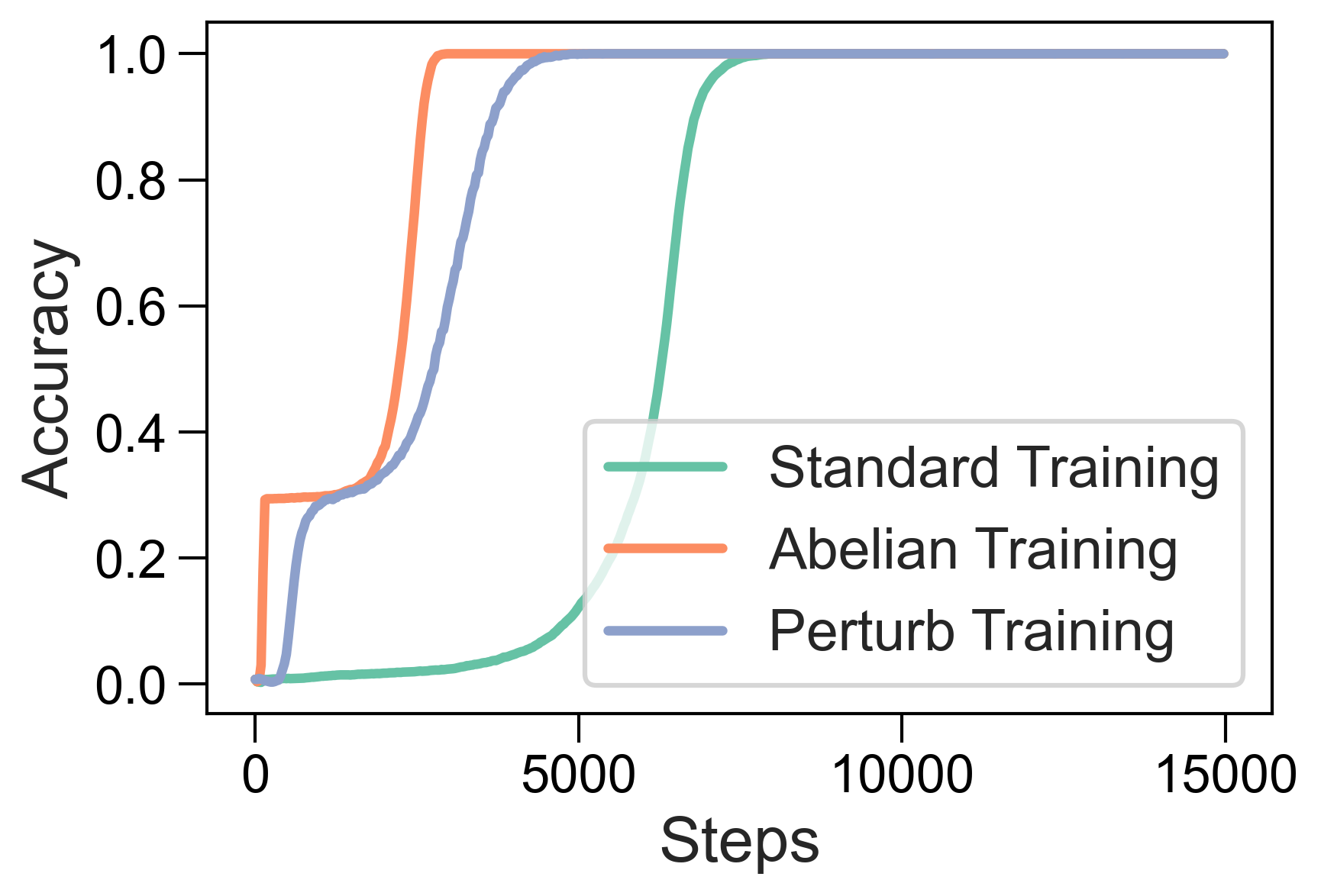}
\caption{Abelian degrok strategy.}
\label{fig: abelian degrok}
\end{figure}

\subsection{Does the robustness-based degrokking strategy change the representation?}

At first glance, the robustness (perturbation)-based degrokking strategy made a minimum change to the initial training. Thus, one may doubt whether it has learned a similar representation as the initial one. In Figure \ref{fig:abelian test} (a) and (b), we can see that the abelian test accuracy of both standard and perturb training is stable after it first reaches a $100 \%$ test accuracy. But as our design of abelian accuracy only implicitly involves representation, we need more detailed quantities to see whether the representation of these two types of training coincides or not.

In Figure \ref{fig:abelian logit test}, we plot the distance of the logits between ``$a+b=$'' and ``$b+a=$'' on samples $(a, b)$s. This quantity depends directly on the representation and we call this ``Abelian test on logits''. Figure \ref{fig:abelian logit test} (a) shows that the logits distance of standard training is stable once it reaches 0, but (b) has a chaotic behavior when the test accuracy reaches $100 \%$. It clearly shows that the representations obtained by standard and perturb training are different, as the dynamic differs a lot.

\begin{figure}[htb]
\centering
\begin{subfigure}[b]{0.495\columnwidth}
\includegraphics[width=\linewidth]{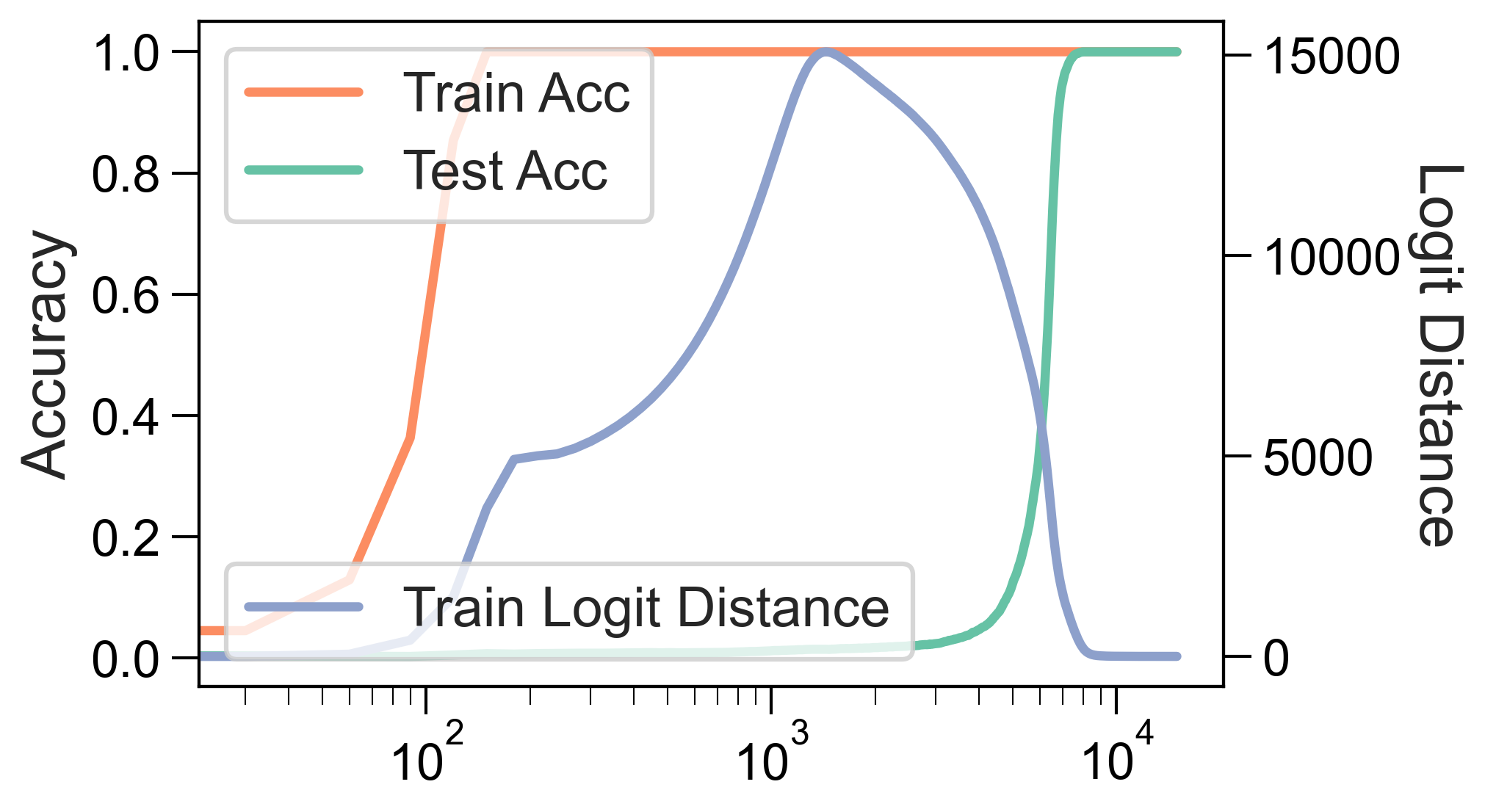}
\caption{Standard training}
\end{subfigure}
\begin{subfigure}[b]{0.495\columnwidth}
\includegraphics[width=0.97\linewidth]{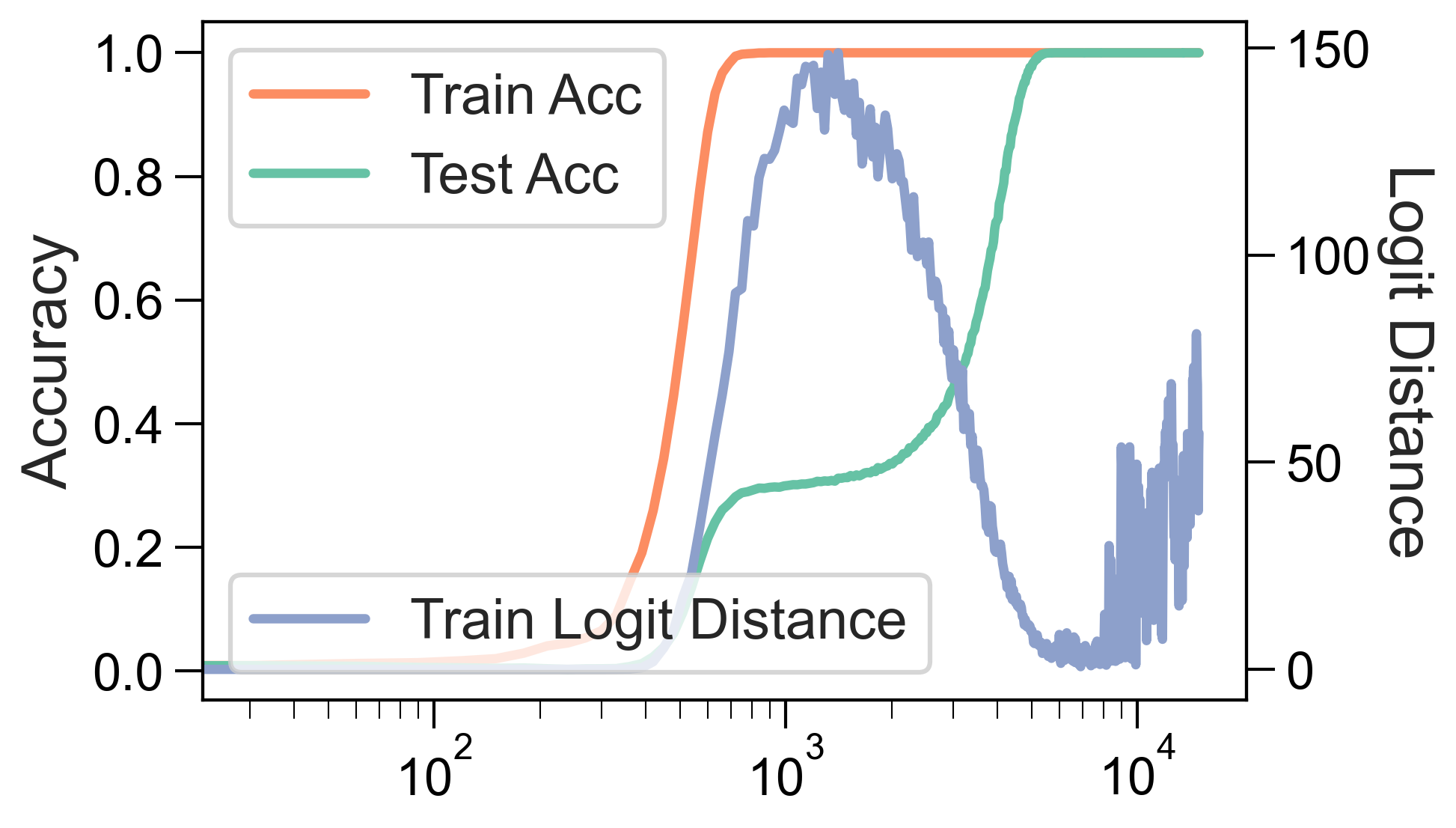}
\caption{Perturb training}
\end{subfigure}
\caption{Abelian test on logits.}
\label{fig:abelian logit test}
\end{figure}

\section{New metrics}

We found that the decay of $l_2$ weight norm is a sufficient condition to indicate grokking on the test dataset (theorem \ref{l2 decay}). One may wonder can we propose new metrics that better correlate with the grokking process? We answer this question in the affirmative by borrowing ideas from robustness and information theory.

\subsection{New metrics that correlate better to test accuracy} \label{metrics}

As we want to have metrics that correlate more closely with the grokking process, motivated by the use of information theory to explain neural network learning \citep{tishby2015deep}, we want to use metrics that are information-theoretically understandable. But as information-theoretic quantities have usually hard to compute exactly, we find that matrix information-theoretic quantities may be better choices as they are simpler to compute. 

Note when computing matrix information-theoretic quantities, we need to satisfy its requirement (i.e. an all $1$ diagonal and positive semi-definiteness). A very easy construction of a matrix satisfying these requirements is the ($l_2$ normalized) feature gram matrix. Note features are evaluated on specific datasets, during training we only have access to the training dataset. As we have discussed the usage of robustness in explaining grokking, we would like the features calculated based on a perturbed training dataset.

We will give the definite of our metrics ``perturbed mutual information'' based on the intuition discussed above. Informally, perturbed mutual information is defined as the (feature $l_2$ normalized gram matrix) mutual information of the input and output layers on dataset $\mathcal{D}_{train}$, when input is perturbed by $\Delta \sim \mathcal{N}(0, \sigma^2 \mathbf{I})$. \textbf{Denote $W=(W_1, W_2)$, where $W_1$ is the first layer of the neural network.} 

\begin{definition}[Perturb Mutual Information] Suppose the perturbation is sampled as $\Delta = (\Delta_i)^n_{i=1} \sim \mathcal{N}(0, \sigma^2 \mathbf{I})$. Denote $\mathbf{Z}_1(\Delta) = [\mathbf{z}^{(1)}_1 \cdots \mathbf{z}^{(1)}_n]$, where $\mathbf{z}^{(1)}_i =\frac{f(\mathbf{x}_i + \Delta_i, W_1 )}{\|f(\mathbf{x}_i + \Delta_i, W_1 )\|}$. Also denote $\mathbf{Z}_2(\Delta) = [\mathbf{z}^{(2)}_1 \cdots \mathbf{z}^{(2)}_n]$, where $\mathbf{z}^{(2)}_i =\frac{f(\mathbf{x}_i + \Delta_i, W )}{\|f(\mathbf{x}_i + \Delta_i, W )\|}$. Define $\mathbf{G}_1(\Delta) = \mathbf{Z}^T_1(\Delta)\mathbf{Z}_1(\Delta)$ and $\mathbf{G}_2(\Delta) = \mathbf{Z}^T_2(\Delta)\mathbf{Z}_2(\Delta)$. The perturb mutual information is defined as:
\begin{equation*}
\operatorname{PMI}(f, \sigma, \alpha)  = \mathbb{E}_{\Delta  \sim \mathcal{N}(0, \sigma^2 \mathbf{I})} \operatorname{I}_{\alpha}(\mathbf{G}_1(\Delta), \mathbf{G}_2(\Delta)).
\end{equation*}
\end{definition}

\begin{figure}[htb]
\centering
\begin{subfigure}[b]{0.49\columnwidth}
\includegraphics[width=\linewidth]{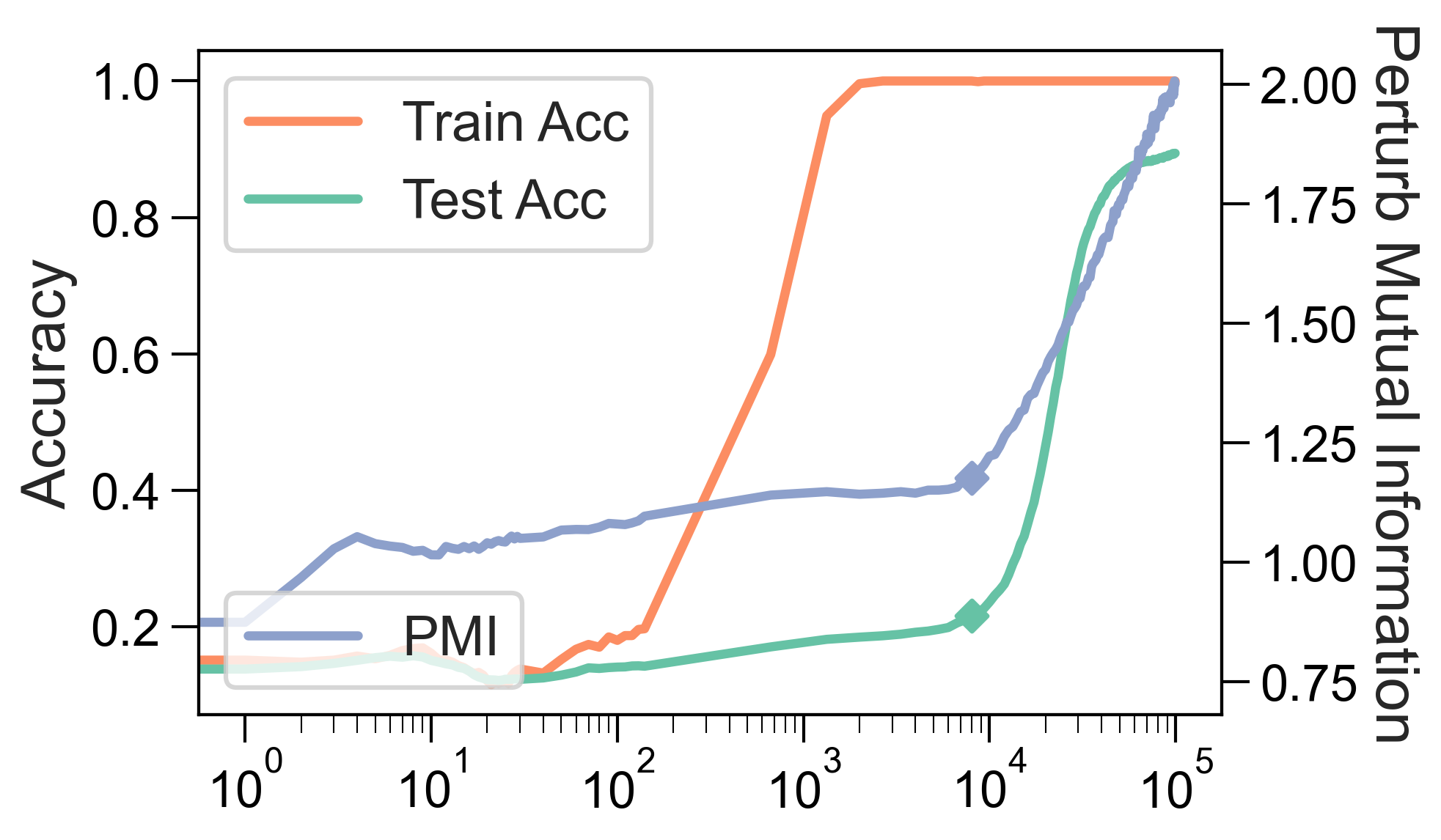}
\caption{MNIST Dataset}
\end{subfigure}
\begin{subfigure}[b]{0.49\columnwidth}
\includegraphics[width= \linewidth]{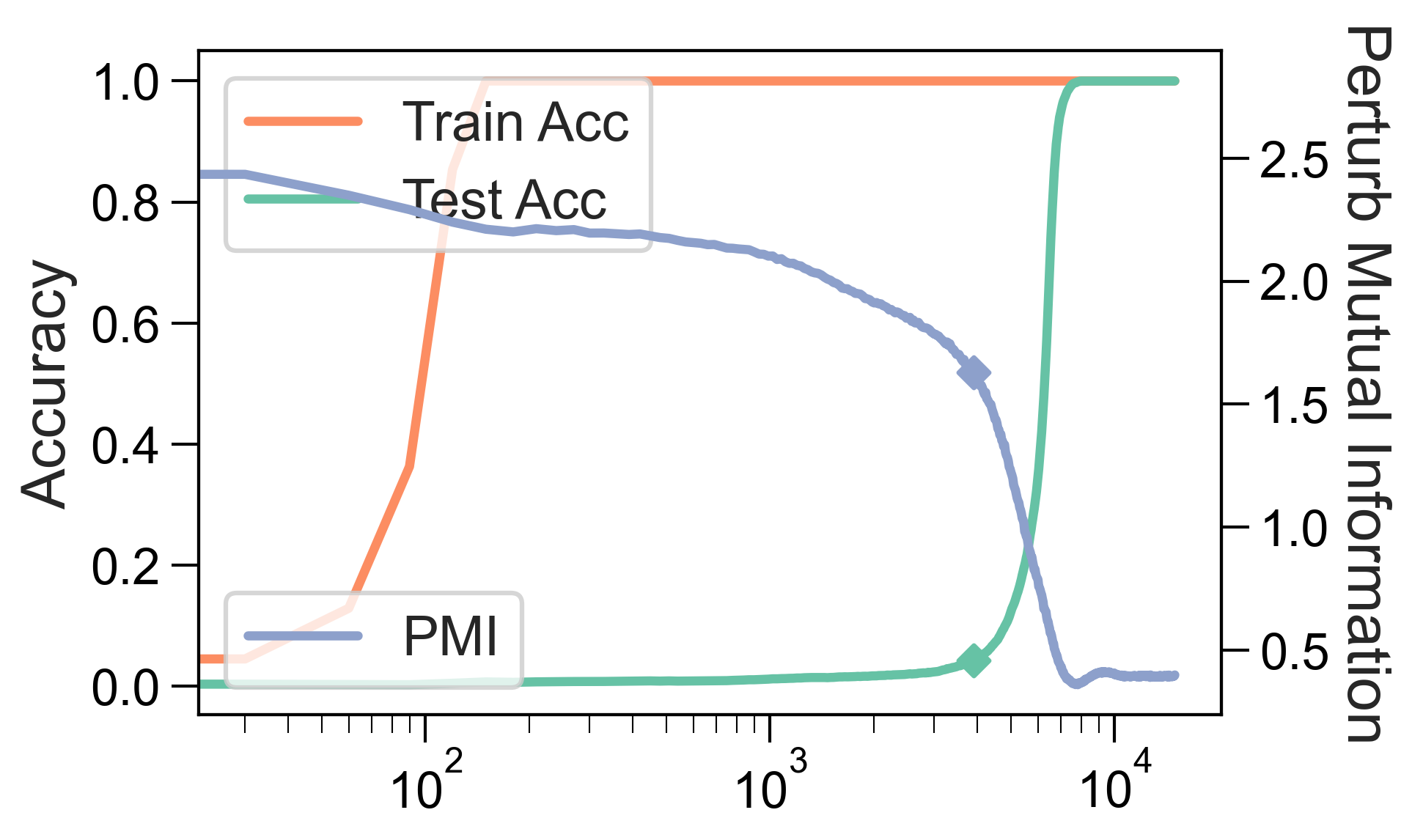}
\caption{Modulo Addition Dataset}
\end{subfigure}
\caption{Perturb mutual information on the training dataset.}
\label{fig:perturb mi}
\end{figure}

Similarly, we define the perturb entropy metric as follows.

\begin{definition}[Perturb Entropy] Suppose the perturbation is sampled as $\Delta = (\Delta_i)^n_{i=1} \sim \mathcal{N}(0, \sigma^2 \mathbf{I})$. Denote $\mathbf{Z}(\Delta) = [\mathbf{z}_1 \cdots \mathbf{z}_n]$, where $\mathbf{z}_i =\frac{f(\mathbf{x}_i + \Delta_i, W )}{\|f(\mathbf{x}_i + \Delta_i, W )\|}$. Define $\mathbf{G}(\Delta) = \mathbf{Z}^T(\Delta)\mathbf{Z}(\Delta)$. The perturb entropy is defined as:
\begin{equation*}
\operatorname{PE}(f, \sigma, \alpha)  = \mathbb{E}_{\Delta  \sim \mathcal{N}(0, \sigma^2 \mathbf{I})} \operatorname{H}_{\alpha}(\mathbf{G}(\Delta)).
\end{equation*}
\end{definition}

From the definition of perturb mutual information and entropy, we can see it has a relatively high computation overhead. Therefore, we will make a few approximations to accelerate the computing. The parameter $\alpha$ is usually set as $1$ as this closely aligns with the classical Shannon entropy. As the number of samples is relatively high, we instead calculate the quantity within each batch and the average among batches. Note for ease of calculation, we sample $\Delta$ only once. 

We plot the perturb mutual information (PMI) in Figure \ref{fig:perturb mi} (a) and (b), it is clear that perturb mutual information changes sharply at the time test accuracy undergoes a quick increase. On MNIST, $\sigma=0.1$. On Modulo Addition Dataset, $\sigma=0.4$.

We plot the perturb entropy (PE) in Figure \ref{fig:perturb entropy} (a) and (b), it is clear that perturb entropy changes sharply at the time test accuracy undergoes a quick increase.

From the above observations, we find that the proposed two metrics PMI and PE change sharply \emph{at the time} of grokking. This makes them better indicators for grokking than the $l_2$ weight norm.

\begin{figure}[htb]
\centering
\begin{subfigure}[b]{0.49\columnwidth}
\includegraphics[width=\linewidth]{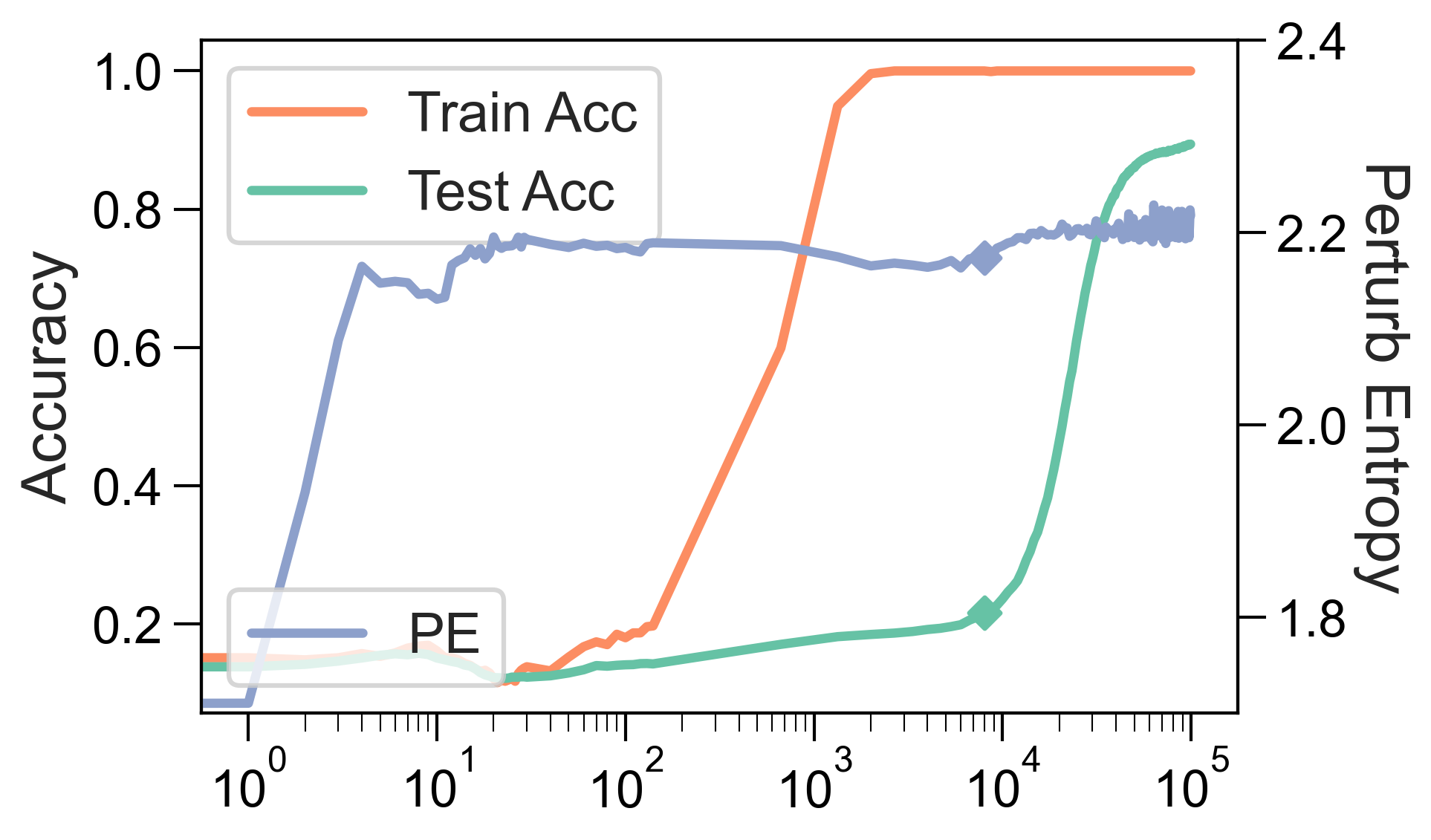}
\caption{MNIST Dataset}
\end{subfigure}
\begin{subfigure}[b]{0.49\columnwidth}
\includegraphics[width=\linewidth]{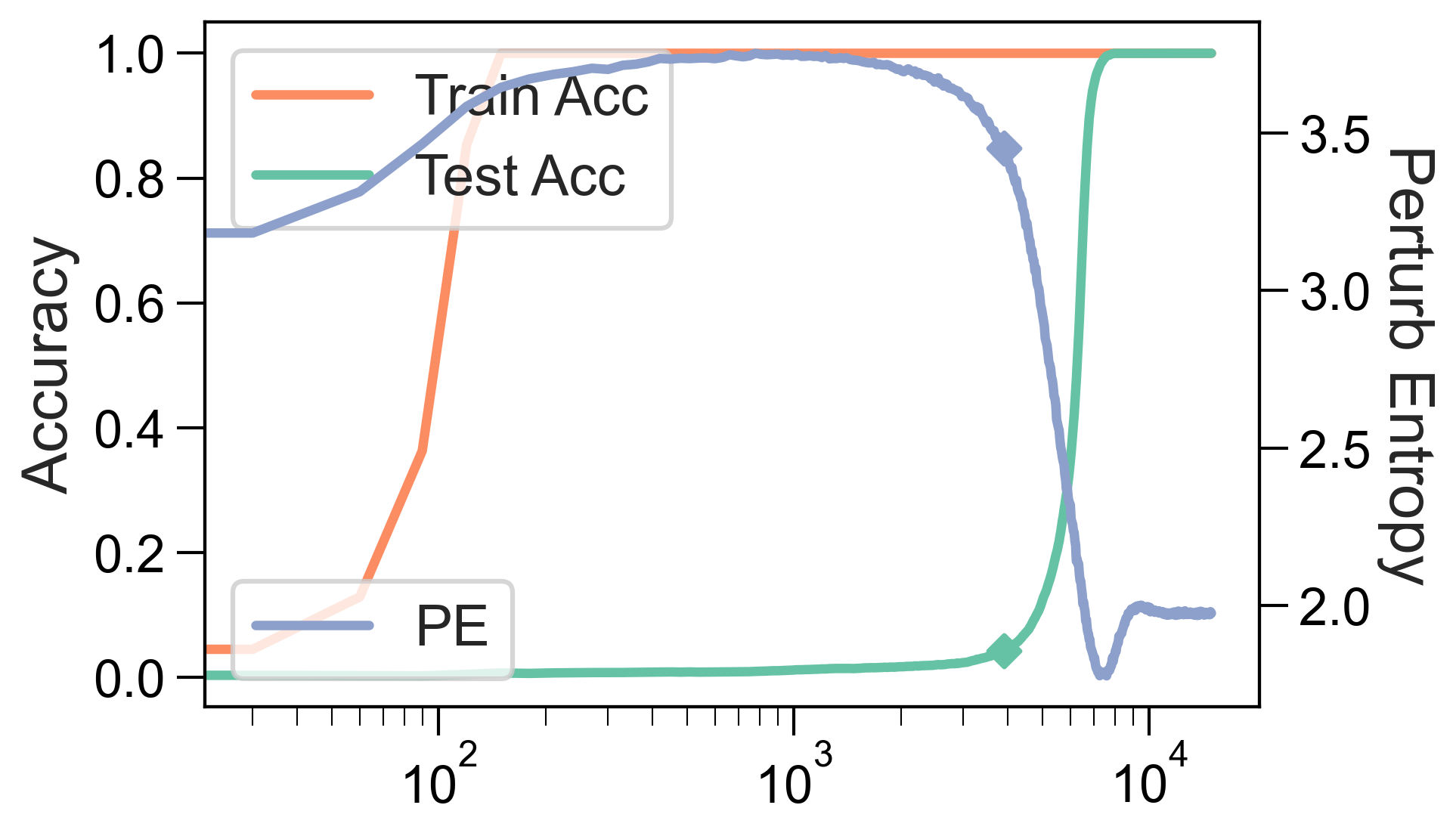}
\caption{Modulo Addition Dataset}
\end{subfigure}
\caption{Perturb entropy on the training dataset.}
\label{fig:perturb entropy}
\end{figure}

\subsection{Predict grokking}

As the metrics PMI and PE closely correlate with the grokking process. One may wonder can these metrics help predict whether a learning process will eventually grok or not by observing the dynamics in the early stage of training. Motivated by the fact that robustness is closely related to the grokking process, we find that the difference between perturbed and un-perturbed information-theoretic quantities may help us identify processes that will eventually grok.

\begin{definition}
(Mutual Information Difference) $\text{MID}(f, \sigma, \alpha)= | \text{PMI}(f, \sigma, \alpha) - \text{PMI}(f, 0, \alpha) |$.
\end{definition}

\begin{definition}
(Entropy Difference) $\text{ED}(f, \sigma, \alpha)= | \text{PE}(f, \sigma, \alpha) - \text{PE}(f, 0, \alpha) |$.
\end{definition}

As we are caring about ``predicting'', this means that we should focus only on the early stage training. Therefore we are considering the training process \textbf{before} the training accuracy reaches $100 \%$. We set $\alpha$ and $\sigma$ just as section \ref{metrics}. In Figure \ref{fig:predict mi} and \ref{fig:predict entropy}, we can see when both MID or ED have a phase of consistent decreasing. Additionally, on MNIST, the MID and ED are very small throughout the whole training process.

\begin{figure}[htb]
\centering
\begin{subfigure}[b]{0.49\columnwidth}
\includegraphics[width=\linewidth]{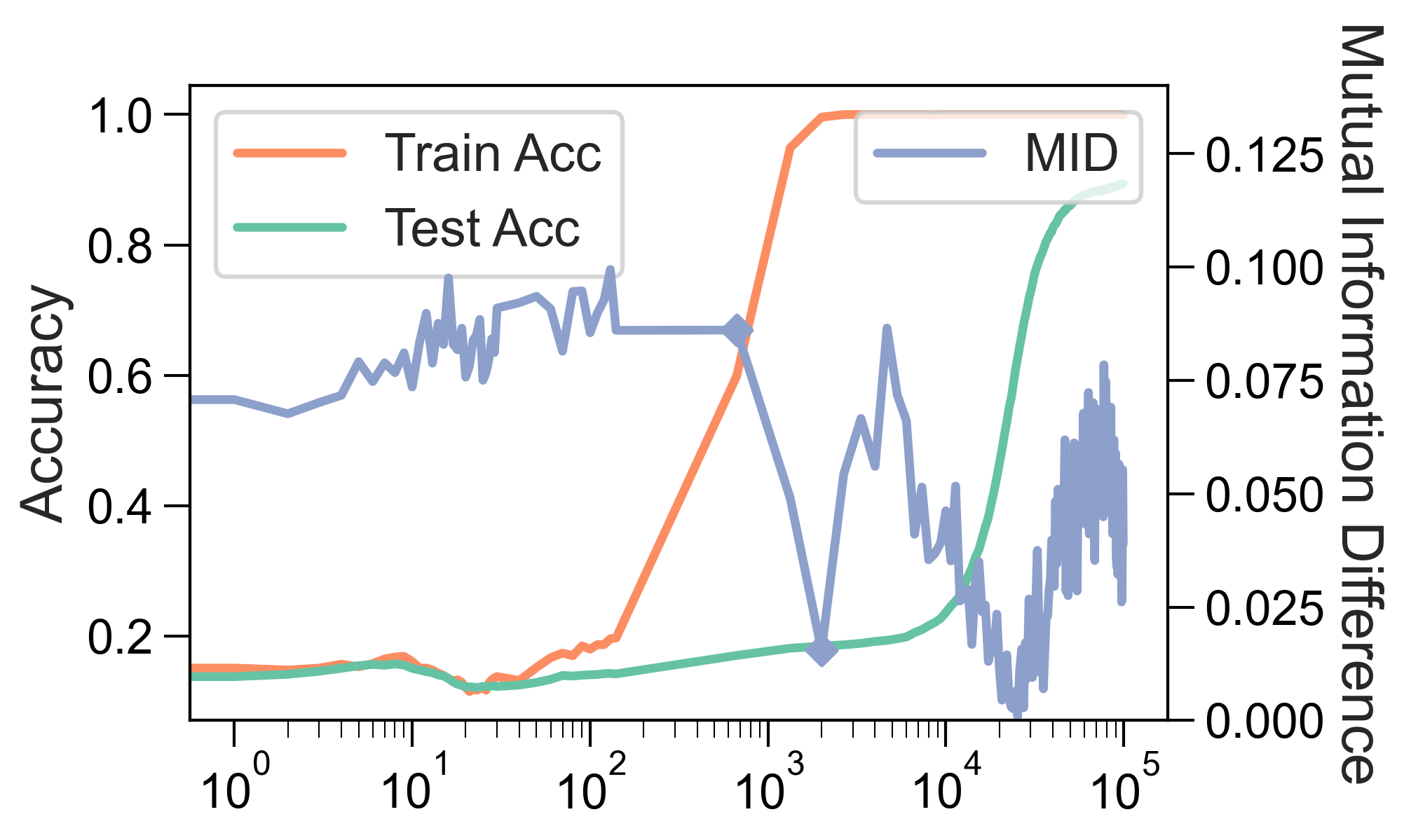}
\caption{MNIST Dataset}
\end{subfigure}
\begin{subfigure}[b]{0.49\columnwidth}
\includegraphics[width= \linewidth]{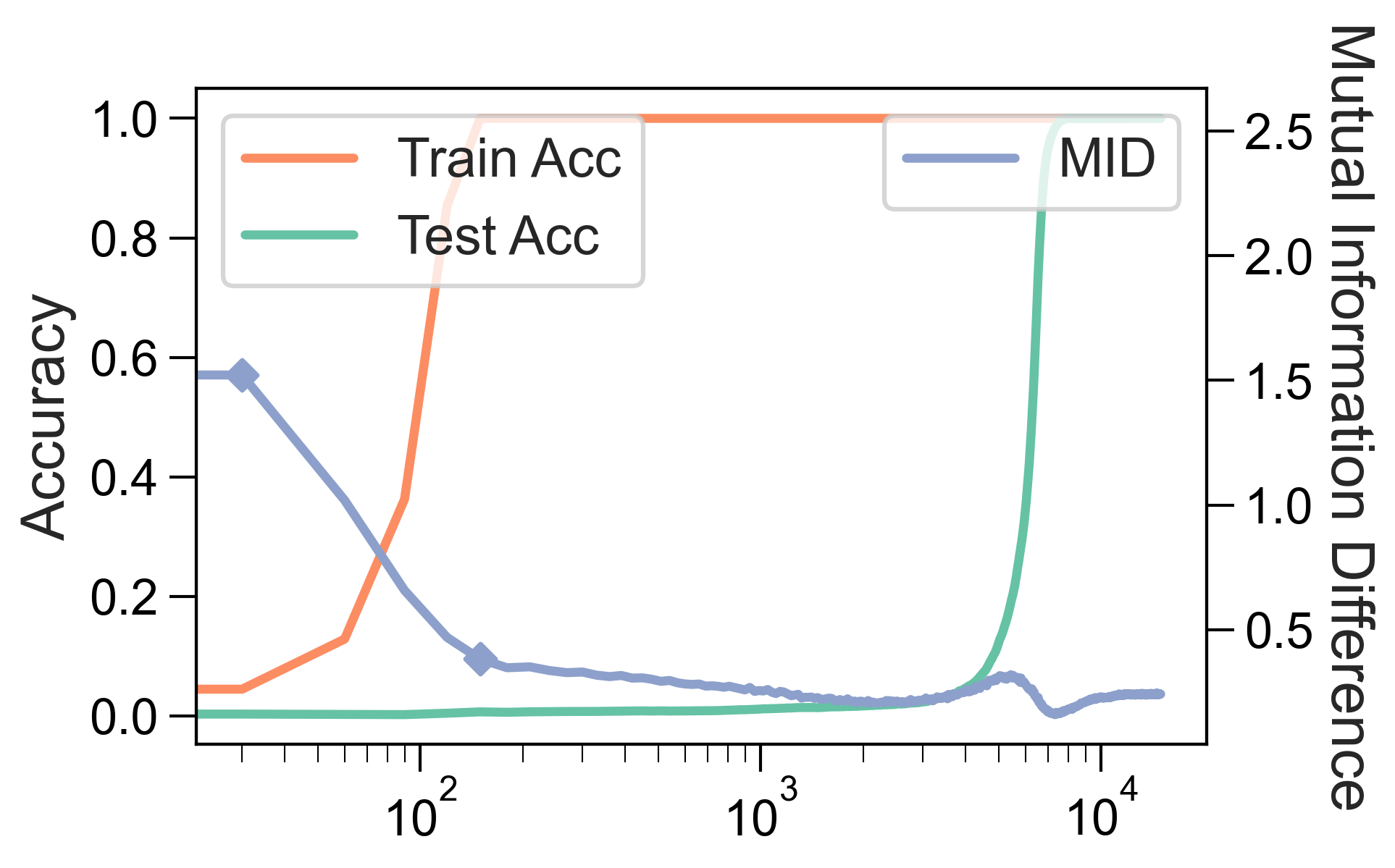}
\caption{Modulo Addition Dataset}
\end{subfigure}
\caption{Grokking and mutual information difference.}
\label{fig:predict mi}
\end{figure}

\begin{figure}[htb]
\centering
\begin{subfigure}[b]{0.49\columnwidth}
\includegraphics[width=\linewidth]{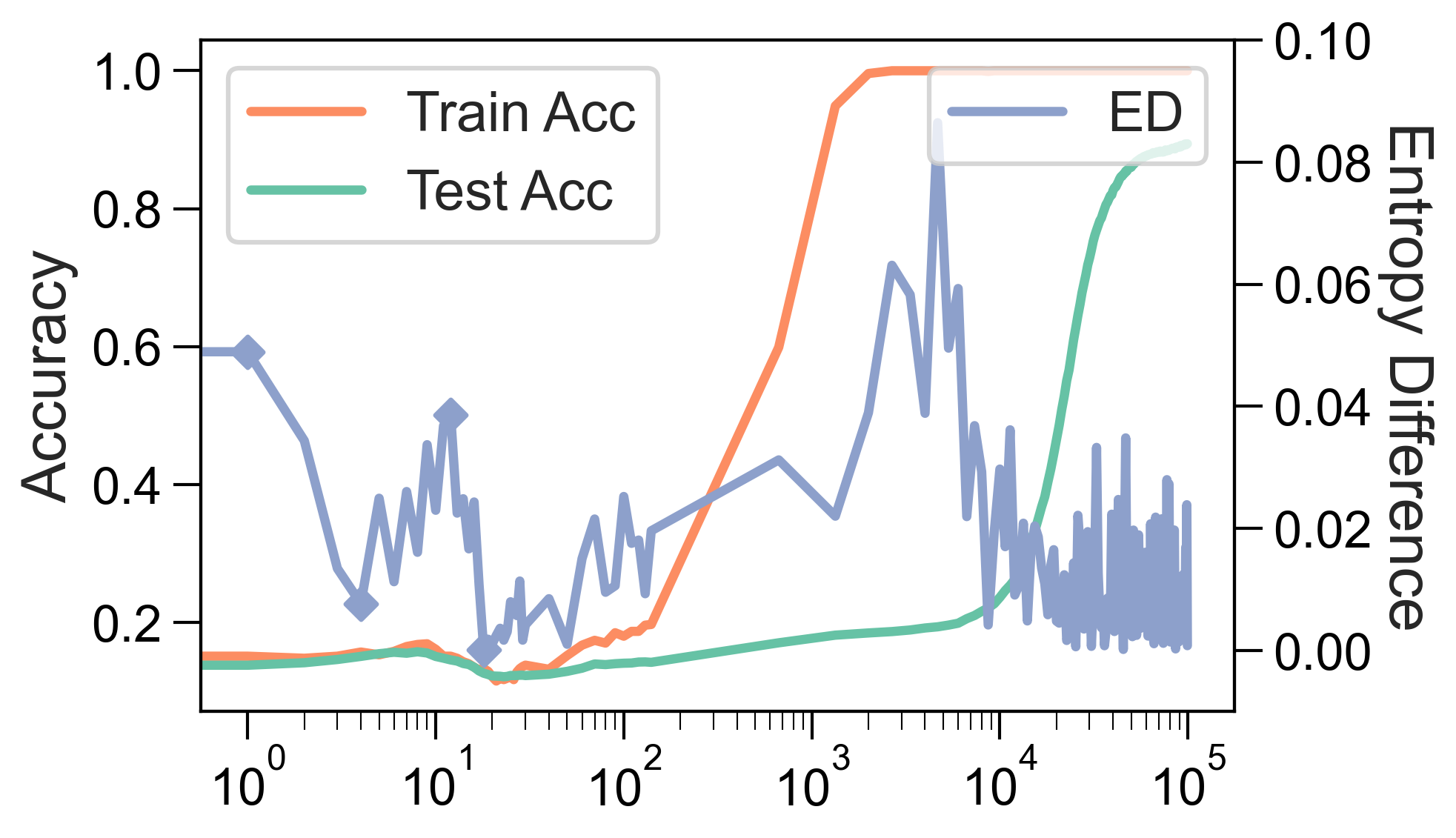}
\caption{MNIST Dataset}
\end{subfigure}
\begin{subfigure}[b]{0.49\columnwidth}
\includegraphics[width= 0.95\linewidth]{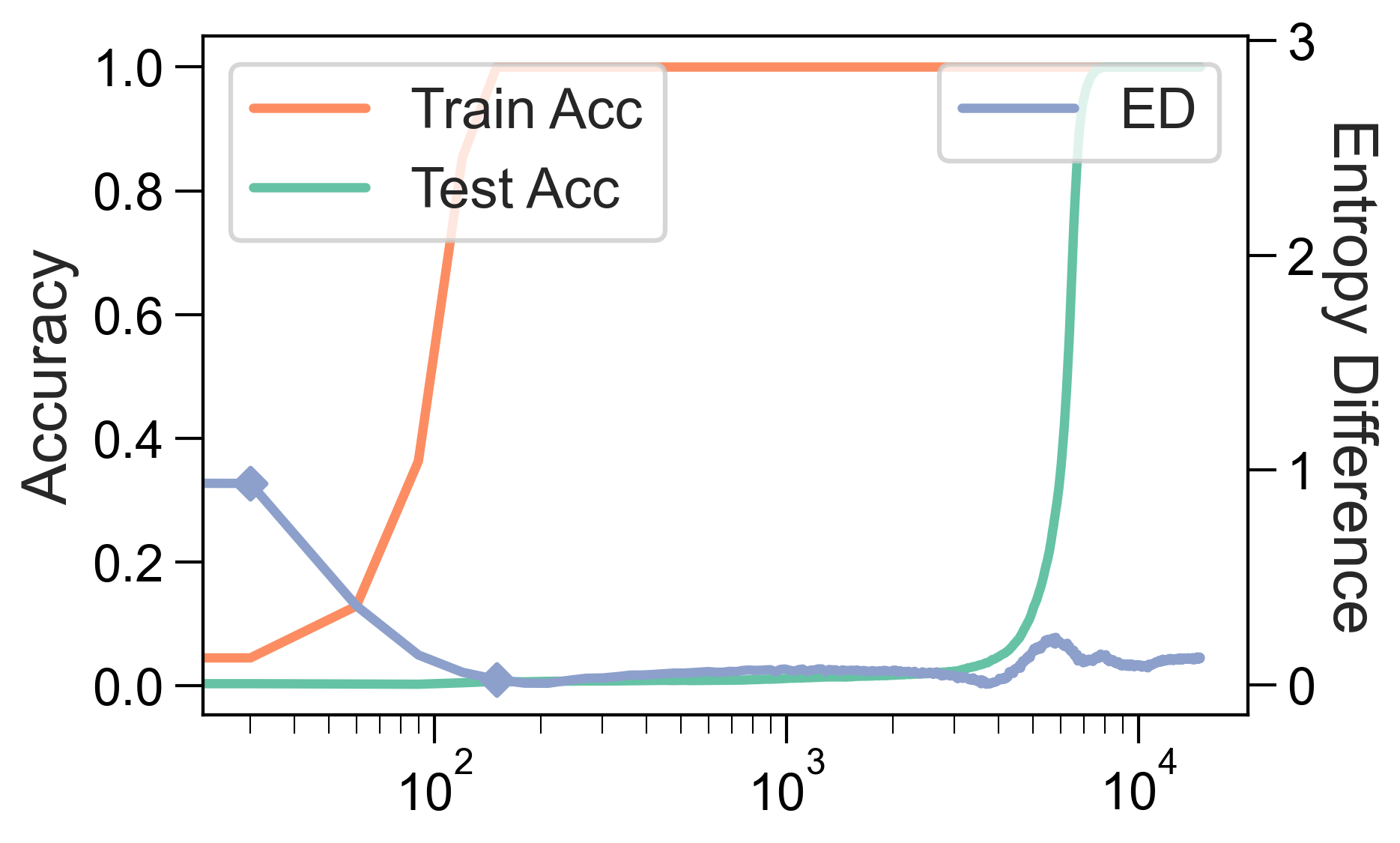}
\caption{Modulo Addition Dataset}
\end{subfigure}
\caption{Grokking and entropy difference.}
\label{fig:predict entropy}
\end{figure}

From the experiments, we thus hypothesize that grokking happens when MID and ED are very small or undergo a sharp decreasing phase \emph{before} training accuracy reaches $100 \%$. This has an intuitive explanation: As we made small perturbations in the input, the resulting perturbation in representations will reflect the robustness of the network. Thus MID and ED will reflect the robustness of the network and we have shown robustness closely correlates with grokking.

We will then prove that the above indexes (MID and ED) may indeed indicate grokking from a robustness viewpoint. In this subsection, we assume $\alpha=2$ in the matrix information quantities to simplify the proof. This implication is based on the fact that we find that matrix information quantities have tendencies robust to the choice of $\alpha$ and we do ablation studies in section \ref{ablation}.

The following two lemmas \citep{tan2023information} show that mutual information and entropy can be transformed into norms, which is more related in the context of robustness.

\begin{proposition} \label{mutual information reduction}
$\operatorname{I}_2(\mathbf{R}_1; \mathbf{R}_2) = 2\log n - \log \frac{|| \mathbf{R}_1 ||^2_F || \mathbf{R}_2 ||^2_F}{|| \mathbf{R}_1 \odot \mathbf{R}_2 ||^2_F}$, where $n$ is the size of matrix $\mathbf{R}_1$ and $\mathbf{R}_2$ and $F$ is the Frobenius norm.
\end{proposition}

\begin{proposition} \label{joint entropy reduction}
Suppose $\mathbf{K}\in \mathbb{R}^{n \times n}$. Then
$\operatorname{H}_2(\mathbf{K}) = 2 \log n - \log \mid \mid \mathbf{K}  \mid \mid^2_F$, where $F$ is the Frobenius norm.
\end{proposition}

We will show that the more robust a network is, the smaller the indexes (MID and ED) we defined above. We will discuss the small differences in the representations that will result in small perturbation of each element in the gram matrix, thus small distances in the Frobenius norm.

The following lemma shows that the perturbation of the gram matrix can be controlled by the perturbation of representations.
\begin{lemma} \label{gram perturb}
Suppose vectors $\mathbf{a}_i$, $\mathbf{a}_j$ and $\mathbf{b}_i$, $\mathbf{b}_j$ are $l_2$ normalized, then $| \langle \mathbf{a}_i, \mathbf{a}_j \rangle - \langle \mathbf{b}_i, \mathbf{b}_j \rangle | \leq 4 (\|\mathbf{a}_i- \mathbf{b}_i \| + \|\mathbf{a}_j- \mathbf{b}_j \| )$.   
\end{lemma}

Using lemmas \ref{joint entropy reduction} and \ref{gram perturb}, we can obtain the following bounds on the difference of entropy.

\begin{lemma} \label{entropy difference}
Suppose $\mathbf{Z} = [\mathbf{z}_1 \cdots \mathbf{z}_n]$ and $\mathbf{Z}' = [\mathbf{z}'_1 \cdots \mathbf{z}'_n]$ have each of their columns $l_2$ normalized. Denote $\mathbf{G} = \mathbf{Z}^T \mathbf{Z}$ and $\mathbf{G}' = (\mathbf{Z}')^{T} \mathbf{Z}'$. Then we have following inequality:
\begin{equation*}
| \operatorname{H}_2(\mathbf{G}) -  \operatorname{H}_2(\mathbf{G}') | \leq 8 \sum^n_{i=1} \| \mathbf{z}_i - \mathbf{z}'_i \|.
\end{equation*}
\end{lemma}

Then it is straightforward to obtain the following bound on ED using lemma \ref{entropy difference}.

\begin{theorem} \label{ED bound}
The entropy difference can be bounded by the perturbation of network representation as follows:
\begin{align*}
&\operatorname{ED}(f, \sigma, 2) \\ & \leq 8 \mathbb{E}_{\Delta \sim \mathcal{N}(0, \sigma^2 \mathbf{I})} \sum^n_{i=1}  \|f(\mathbf{x}_i + \Delta_i, W) - f(\mathbf{x}_i, W)  \|. 
\end{align*}
\end{theorem}

The calculation of mutual information involves Hadarmard product. The following lemma bounds the perturbation of Hadarmard product.
\begin{lemma} \label{Hadarmard}
If numbers $| a |$ $| a' 
 |$ and $| b |$ $| b' |$ all less or equal than 1, then $| ab - a' b' | \leq | a - a'| + | b - b' |$.    
\end{lemma}

Using lemma \ref{mutual information reduction}, \ref{gram perturb}, \ref{entropy difference}, and \ref{Hadarmard}, we can obtain the following bounds on the difference of mutual information.

\begin{lemma}\label{mutual information difference}
Suppose $\mathbf{Z}_1 = [\mathbf{z}^{(1)}_1 \cdots \mathbf{z}^{(1)}_n]$, $\mathbf{Z}_2 = [\mathbf{z}^{(2)}_1 \cdots \mathbf{z}^{(2)}_n]$ and $\mathbf{Z}'_1 = [(\mathbf{z}'_1)^{(1)} \cdots (\mathbf{z}'_n)^{(1)}]$, $\mathbf{Z}'_2 = [(\mathbf{z}'_1)^{(2)} \cdots (\mathbf{z}'_n)^{(2)}]$ have each of their columns $l_2$ normalized. Denote $\mathbf{G}_1 = \mathbf{Z}^T_1 \mathbf{Z}_1$, $\mathbf{G}_2 = \mathbf{Z}^T_2 \mathbf{Z}_2$ and $\mathbf{G}'_1 = (\mathbf{Z}'_1)^T \mathbf{Z}'_1$, $\mathbf{G}'_2 = (\mathbf{Z}'_2)^T \mathbf{Z}'_2$. Then we have the following inequality:
\begin{equation*}
| \operatorname{I}_2(\mathbf{G}_1, \mathbf{G}_2) -  \operatorname{I}_2(\mathbf{G}'_1, \mathbf{G}'_2) | \leq 16 \sum^2_{j=1} \sum^n_{i=1} \| \mathbf{z}^{(j)}_i - (\mathbf{z}'_i)^{(j)} \|.
\end{equation*}    
\end{lemma}

Then it is straightforward to obtain the following bound on MID using lemma \ref{mutual information difference}.

\begin{theorem} \label{MID bound}
The entropy difference can be bounded by the perturbation of network representation as follows:
\begin{align*}
&\operatorname{MID}(f, \sigma, 2) \leq 16 \mathbb{E}_{\Delta \sim \mathcal{N}(0, \sigma^2 \mathbf{I})}  \sum^n_{i=1} (  \|f(\mathbf{x}_i + \Delta_i, W) \\ & - f(\mathbf{x}_i, W)  \| + \|f(\mathbf{x}_i + \Delta_i, W_1) - f(\mathbf{x}_i, W_1)  \| ) , 
\end{align*}
recall that $W_1$ is the first layer of the network.
\end{theorem}

Theorem \ref{ED bound} and \ref{MID bound} clearly show that MID and ED have a close relationship with network robustness. Therefore making them good candidates for predicting grokking.

%%%%%%%%%%%%%%%%%%%%%%%%%%%%%%
%%% Algorithm

%%%%%%%%%%%%%%%%%%%%%%%%%%%%%%
%%% Experiments

\section{Ablation study} 
\label{ablation}

% \subsection{Ablations on perturb entropy} \label{ablation}

In Figure \ref{fig: MNIST ablation}, we vary the parameter $\alpha$ used in the calculation of matrix entropy, the goal is to show that the tendency used in this paper is robust to the choice of $\alpha$. To further see the difference between train and test datasets. We conduct this on the un-perturbed MNIST test dataset. These quantities have a tendency robust to $\alpha$ and train/test dataset.

\begin{figure}[t] 
\centering 
\includegraphics[width=0.5\columnwidth]{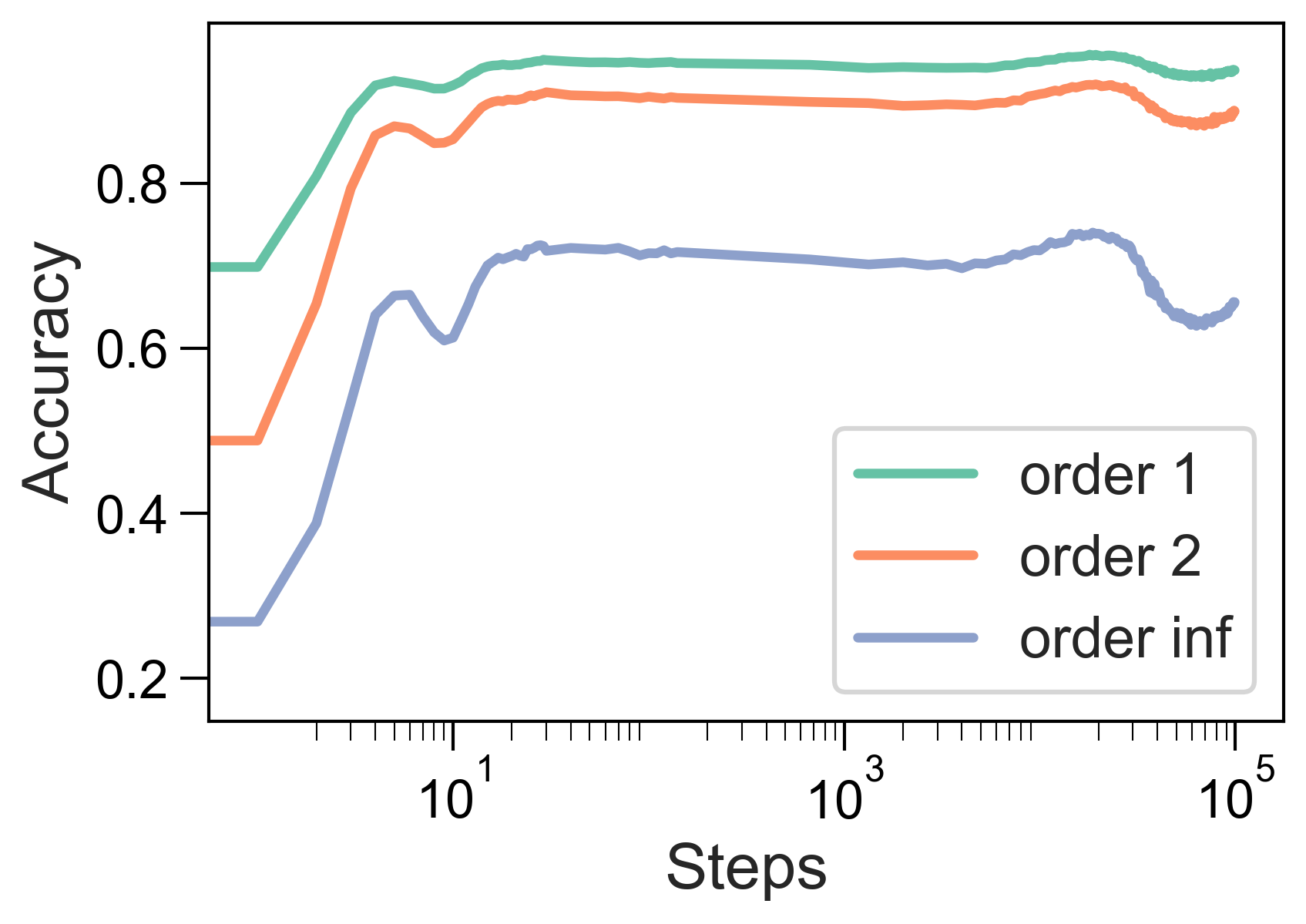}
\caption{Ablation on entropy order.}
\label{fig: MNIST ablation}
\end{figure}

\section{Conclusion}
In this paper, we understood grokking from the perspective of neural network robustness. We found that the $l_2$ weight norm is a sufficient condition for grokking. We then proposed perturbation-based methods to speed up generalization. We also discovered that the standard training process on the Modulo Addition Dataset lacks learning of fundamental group operations before grokking occurs. Interestingly, our method for accelerating generalization involves learning the commutative law, which is necessary for grokking on the test dataset. We introduced new metrics based on robustness and information theory that correlate well with grokking and have the potential to predict it.

%%%%%%%%%%%%%%%%%%%%%%%%%%%%%%
%%% Related
% To DO

\section{Related Works}

\paragraph{Grokking.}
Grokking, originally introduced by \citet{power2022grokking}, refers to the intriguing phenomenon observed in small transformers training on algorithmic tasks. It has been noted that, after an extended period of training, test accuracy experiences a sudden improvement, transitioning from near-random to perfect understanding. Several explanations have been proposed to shed light on this very unusual phenomenon. \citet{nanda2023progress} demonstrate that Fourier analysis can provide insights into understanding grokking. In contrast, \citet{barak2022hidden} hypothesize that the network steadily progresses towards generalization rather than making abrupt leaps. Furthermore, smaller instances of grokking have been explored by \citet{liu2022towards}, unveiling distinct learning phases. Importantly, \citet{thilak2022slingshot} argue that grokking can manifest without explicit regularization, proposing the existence of an implicit regularizer known as the slingshot mechanism. Furthermore, \citet{liu2022omnigrok} demonstrate that grokking is not limited to modulo addition tasks but can occur in more conventional tasks as well.

\paragraph{Information theory in understanding neural networks.}
Information theory has long provided a valuable framework for understanding the interplay between probability and information. It has been used to understand the inner workings of neural networks \citep{tishby2000information, tishby2015deep}, trying to make the mechanism of neural network operations more ``white-box''. As the computation burden of information-theoretic quantities is usually high, a recent departure from traditional information theory has emerged, shifting the focus towards generalizing it for measuring relationships between matrices \citep{bach2022information, skean2023dime, zhang2023relationmatch, zhang2023kernel}. \citet{tan2023information} apply matrix mutual information and joint entropy to self-supervised learning.

%%%%%%%%%%%%%%%%%%%%%%%%%%%%%%
%%% Conclusion

% \input{contents/reproducibility}

\clearpage

\bibliography{reference}
\bibliographystyle{iclr}

%%%%%%%%%%%%%%%%%%%%%%%%%%%%%%%%%%%%%%%%%%%%%%%%%%%%%%%%%%%%%%%%%%%%%%%%%%%%%%%
% APPENDIX
%%%%%%%%%%%%%%%%%%%%%%%%%%%%%%%%%%%%%%%%%%%%%%%%%%%%%%%%%%%%%%%%%%%%%%%%%%%%%%%
\clearpage
\appendix

\textbf{\large Appendix}

\section{Detailed proofs} \label{proofs}

\begin{theorem} 
Suppose $W^*$ is a interpolation solution and the gradient of $f(\mathbf{x}, W^*)$ is $L$-Lipschitz about $\mathbf{x}$. Suppose at least $\delta$-fraction of test data has a train dataset neighbour whose distance is at most $\epsilon(W^*)$, where $\epsilon(W^*) = \min\{1, \frac{1}{2(\sqrt{\frac{n}{\min _i\left\|\mathbf{x}_i\right\|^2_2}\left\|W^*\right\|_F^2 S\left(W^*\right)}+L)}  \}$. Then the test accuracy will be at least $\delta$.  
\end{theorem}

\begin{proof}
For a test data $\mathbf{x}'_i$, denote its training data neighbour as $\mathbf{x}_i$, then 
\begin{align*}
  \| f(\mathbf{x}_i, W^*) -  f(\mathbf{x}'_i, W^*) \| & \leq \| \mathbf{x}_i- \mathbf{x}'_i \| \| \nabla_{\mathbf{x}} f(\mathbf{\hat{x}}, W^*) \| \\
  & \leq \| \mathbf{x}_i- \mathbf{x}'_i \| (\| \nabla_{\mathbf{x}} f(\mathbf{{x}}_i, W^*) \|  + \| \nabla_{\mathbf{x}} f(\mathbf{\hat{x}}, W^*) - \nabla_{\mathbf{x}} f(\mathbf{{x}}_i, W^*) \| )  \\
  & \leq \| \mathbf{x}_i- \mathbf{x}{'}_i \| (\| \nabla_{\mathbf{x}} f(\mathbf{{x}}_i, W^*) \| + L\| \mathbf{x}_i- \mathbf{x}{'}_i \|) \\
  & \leq \| \mathbf{x}_i- \mathbf{x}'_i \| (\| \nabla_{\mathbf{x}} f(\mathbf{{x}}_i, W^*) \| + L)  \\
  & \leq \| \mathbf{x}_i- \mathbf{x}'_i \| ( \sqrt{\frac{n}{\min _i\left\|\mathbf{x}_i\right\|^2_2}\left\|W^*\right\|_F^2 S\left(W^*\right)} + L) \\
  & \leq \frac{1}{2}. 
\end{align*}

Thus the difference between $f(\mathbf{x}_i, W^*)$ and $f(\mathbf{x}'_i, W^*)$ is bounded by $\frac{1}{2}$, as $W^*$ is the interpolating solution. Taking the usual argmax decision rule will make $\mathbf{x}'_i$ successfully classified.
\end{proof}

\begin{lemma}
Suppose vectors $\mathbf{a}_i$, $\mathbf{a}_j$ and $\mathbf{b}_i$, $\mathbf{b}_j$ are $l_2$ normalized, then $| \langle \mathbf{a}_i, \mathbf{a}_j \rangle - \langle \mathbf{b}_i, \mathbf{b}_j \rangle | \leq 4 (\|\mathbf{a}_i- \mathbf{b}_i \| + \|\mathbf{a}_j- \mathbf{b}_j \| )$.   
\end{lemma}

\begin{proof}
\begin{align*}
| \langle \mathbf{a}_i, \mathbf{a}_j \rangle - \langle \mathbf{b}_i, \mathbf{b}_j \rangle | & = | \| \mathbf{a}_i - \mathbf{a}_j \|^2 -  \| \mathbf{b}_i - \mathbf{b}_j \|^2 |\\
& = | (\| \mathbf{a}_i - \mathbf{a}_j \|+ \| \mathbf{b}_i - \mathbf{b}_j \|)(\| \mathbf{a}_i - \mathbf{a}_j \|- \| \mathbf{b}_i - \mathbf{b}_j \|) |  \\
& \leq 4 | \| \mathbf{a}_i - \mathbf{a}_j \|- \| \mathbf{b}_i - \mathbf{b}_j \| | \\
& = 4 | \| \mathbf{a}_i - \mathbf{b}_i + \mathbf{b}_j - \mathbf{a}_j +\mathbf{b}_i - \mathbf{b}_j \|- \| \mathbf{b}_i - \mathbf{b}_j \|    | \\
& \leq 4 (\|\mathbf{a}_i- \mathbf{b}_i \| + \|\mathbf{a}_j- \mathbf{b}_j \| ).
\end{align*}
\end{proof}

\begin{lemma}
Suppose $\mathbf{Z} = [\mathbf{z}_1 \cdots \mathbf{z}_n]$ and $\mathbf{Z}' = [\mathbf{z}'_1 \cdots \mathbf{z}'_n]$ have each of their columns $l_2$ normalized. Denote $\mathbf{G} = \mathbf{Z}^T \mathbf{Z}$ and $\mathbf{G}' = (\mathbf{Z}')^{T} \mathbf{Z}'$. Then we have following inequality:
\begin{equation*}
| \operatorname{H}_2(\mathbf{G}) -  \operatorname{H}_2(\mathbf{G}') | \leq 8 \sum^n_{i=1} \| \mathbf{z}_i - \mathbf{z}'_i \|.
\end{equation*}
\end{lemma}

\begin{proof}
\begin{align*}
| \operatorname{H}_2(\mathbf{G}_1) -  \operatorname{H}_2(\mathbf{G}'_1) | & = | \log \| \mathbf{G}_1  \|^2_F - \log \| \mathbf{G}'_1  \|^2_F  | \\
& = | \log \frac{\| \mathbf{G}_1  \|^2_F}{\| \mathbf{G}'_1  \|^2_F}|  \\
& =  \log (1 +  \frac{| \| \mathbf{G}_1  \|^2_F - \| \mathbf{G}'_1  \|^2_F |}{\min \{ \| \mathbf{G}_1  \|^2_F, \| \mathbf{G}'_1  \|^2_F  \} } ) \\
& \leq  \log (1 +  \frac{| \| \mathbf{G}_1  \|^2_F - \| \mathbf{G}'_1  \|^2_F |}{n} ) \\
& \leq \frac{| \| \mathbf{G}_1  \|^2_F - \| \mathbf{G}'_1  \|^2_F |}{n} \\
&= \frac{|\sum_{i,j}  \langle \mathbf{z}_i, \mathbf{z}_j \rangle^2 - \langle \mathbf{z}'_i, \mathbf{z}'_j \rangle^2  |}{n} \\
& \leq \frac{\sum_{i,j} | (\langle \mathbf{z}_i, \mathbf{z}_j \rangle - \langle \mathbf{z}'_i, \mathbf{z}'_j \rangle)(\langle \mathbf{z}_i, \mathbf{z}_j \rangle + \langle \mathbf{z}'_i, \mathbf{z}'_j \rangle)  |}{n} \\
&\leq 2 \frac{\sum_{i,j} | \langle \mathbf{z}_i, \mathbf{z}_j \rangle - \langle \mathbf{z}'_i, \mathbf{z}'_j \rangle |}{n} \\
&\leq 8 \frac{\sum_{i \neq j} \|\mathbf{z}_i- \mathbf{z}'_i \| + \|\mathbf{z}_j- \mathbf{z}'_j \|}{n} \\
&= 8 \sum^n_{i=1} \| \mathbf{z}_i - \mathbf{z}'_i \|.
\end{align*}
\end{proof}

\begin{lemma}
if numbers $\mid a \mid$ $\mid a' 
 \mid$ and $\mid b \mid $ $\mid b' \mid$ all less or equal than 1, then $\mid ab - a' b' \mid \leq \mid a - a'\mid + \mid b - b' \mid$.    
\end{lemma}
\begin{proof}
Note $ab - a' b' = (a - a')b + a'(b-b')$. Then the conclusion follows from the triangular inequality.
\end{proof}

\begin{lemma}
Suppose $\mathbf{Z}_1 = [\mathbf{z}^{(1)}_1 \cdots \mathbf{z}^{(1)}_n]$, $\mathbf{Z}_2 = [\mathbf{z}^{(2)}_1 \cdots \mathbf{z}^{(2)}_n]$ and $\mathbf{Z}'_1 = [(\mathbf{z}'_1)^{(1)} \cdots (\mathbf{z}'_n)^{(1)}]$, $\mathbf{Z}'_2 = [(\mathbf{z}'_1)^{(2)} \cdots (\mathbf{z}'_n)^{(2)}]$ have each of their columns $l_2$ normalized. Denote $\mathbf{G}_1 = \mathbf{Z}^T_1 \mathbf{Z}_1$, $\mathbf{G}_2 = \mathbf{Z}^T_2 \mathbf{Z}_2$ and $\mathbf{G}'_1 = (\mathbf{Z}'_1)^T \mathbf{Z}'_1$, $\mathbf{G}'_2 = (\mathbf{Z}'_2)^T \mathbf{Z}'_2$. Then we have the following inequality:
\begin{equation*}
| \operatorname{I}_2(\mathbf{G}_1, \mathbf{G}_2) -  \operatorname{I}_2(\mathbf{G}'_1, \mathbf{G}'_2) | \leq 16 \sum^2_{j=1} \sum^n_{i=1} \| \mathbf{z}^{(j)}_i - (\mathbf{z}'_i)^{(j)} \|.
\end{equation*}    
\end{lemma}

\begin{proof}
Note $|\operatorname{I}_2(\mathbf{G}_1, \mathbf{G}_2) -  \operatorname{I}_2(\mathbf{G}'_1, \mathbf{G}'_2)| = |(\operatorname{H}_2(\mathbf{G}_1) - \operatorname{H}_2(\mathbf{G}'_1)) + (\operatorname{H}_2(\mathbf{G}_2) - \operatorname{H}_2(\mathbf{G}'_2)) + (\operatorname{H}_2(\mathbf{G}'_1 \odot \mathbf{G}'_2) - \operatorname{H}_2(\mathbf{G}_1 \odot \mathbf{G}_2))| \leq |\operatorname{H}_2(\mathbf{G}_1) - \operatorname{H}_2(\mathbf{G}'_1) | + |\operatorname{H}_2(\mathbf{G}_2) - \operatorname{H}_2(\mathbf{G}'_2)| + |\operatorname{H}_2(\mathbf{G}'_1 \odot \mathbf{G}'_2) - \operatorname{H}_2(\mathbf{G}_1 \odot \mathbf{G}_2)| \leq |\operatorname{H}_2(\mathbf{G}'_1 \odot \mathbf{G}'_2) - \operatorname{H}_2(\mathbf{G}_1 \odot \mathbf{G}_2)| + 8 \sum^2_{j=1} \sum^n_{i=1} \| \mathbf{z}^{(j)}_i - (\mathbf{z}'_i)^{(j)} \|.$   

\begin{align*}
&| \operatorname{H}_2(\mathbf{G}'_1 \odot \mathbf{G}'_2) - \operatorname{H}_2(\mathbf{G}_1 \odot \mathbf{G}_2) | \\
& = | \log \| \mathbf{G}'_1 \odot \mathbf{G}'_2  \|^2_F - \log \| \mathbf{G}_1 \odot \mathbf{G}_2  \|^2_F  | \\
& = | \log \frac{\| \mathbf{G}'_1 \odot \mathbf{G}'_2  \|^2_F}{\| \mathbf{G}_1 \odot \mathbf{G}_2  \|^2_F}|  \\
& =  \log (1 +  \frac{| \| \mathbf{G}'_1 \odot \mathbf{G}'_2  \|^2_F - \| \mathbf{G}_1 \odot \mathbf{G}_2  \|^2_F |}{\min \{ \| \mathbf{G}'_1 \odot \mathbf{G}'_2  \|^2_F, \| \mathbf{G}_1 \odot \mathbf{G}_2  \|^2_F  \} } ) \\
& \leq  \log (1 +  \frac{| \| \mathbf{G}'_1 \odot \mathbf{G}'_2  \|^2_F - \| \mathbf{G}_1 \odot \mathbf{G}_2  \|^2_F |}{n} ) \\
& \leq \frac{| \| \mathbf{G}'_1 \odot \mathbf{G}'_2  \|^2_F - \| \mathbf{G}_1 \odot \mathbf{G}_2  \|^2_F |}{n} \\
&= \frac{|\sum_{i,j} (\mathbf{G}'_1(i,j)\mathbf{G}'_2(i,j))^2 - (\mathbf{G}_1(i,j)\mathbf{G}_2(i,j))^2 |}{n} \\
& \leq \frac{\sum_{i,j} | (\mathbf{G}'_1(i,j)\mathbf{G}'_2(i,j)-\mathbf{G}_1(i,j)\mathbf{G}_2(i,j))(\mathbf{G}'_1(i,j)\mathbf{G}'_2(i,j)+\mathbf{G}_1(i,j)\mathbf{G}_2(i,j))  |}{n} \\
&\leq 2 \frac{\sum_{i,j} | \mathbf{G}'_1(i,j)\mathbf{G}'_2(i,j)-\mathbf{G}_1(i,j)\mathbf{G}_2(i,j) |}{n} \\
&\leq 2 \frac{\sum_{i,j} | \mathbf{G}'_1(i,j)-\mathbf{G}_1(i,j) | + | \mathbf{G}'_2(i,j)-\mathbf{G}_2(i,j) |}{n} \\
&\leq 8 \frac{\sum_{i \neq j} \|\mathbf{z}^{(1)}_i- (\mathbf{z}'_i)^{(1)} \| + \|\mathbf{z}^{(1)}_j- (\mathbf{z}'_j)^{(1)} \| + \|\mathbf{z}^{(2)}_i- (\mathbf{z}'_i)^{(2)} \| + \|\mathbf{z}^{(2)}_j- (\mathbf{z}'_j)^{(2)} \|}{n} \\
&= 8 \sum^2_{j=1} \sum^n_{i=1} \| \mathbf{z}^{(j)}_i - (\mathbf{z}'_i)^{(j)} \|.
\end{align*}
\end{proof}

It is interesting to see that if the input gram matrix is $\mathbf{I}$, then $\operatorname{I}_{\alpha}(\mathbf{I}, \mathbf{G}) = \operatorname{H}_{\alpha}(\mathbf{G})$. Note this is usually the case when considering raw pixel picture gram matrix, making mutual information a more 'broadly' quantity.

\section{Relation with information bottleneck} 

In Figure \ref{fig: MNIST MI ablation}, we plot the matrix mutual information of perturbed logits and one-hot labels gram matrices on MNIST, which we termed PMI$'$. We find it is not as timely as PMI to grokking. Note PMI and PMI$'$ can be seen as matrix information versions of the quantities used in information bottleneck \citep{tishby2000information, tishby2015deep}. If we view input as variable $X$, logits as $Z$, and labels as $Y$. Then PMI is similar to $\operatorname{MI}(X,Z)$ and PMI$'$ is similar to $\operatorname{MI}(Z,Y)$.

\begin{figure}[t] 
\centering 
\includegraphics[width=0.5\columnwidth]{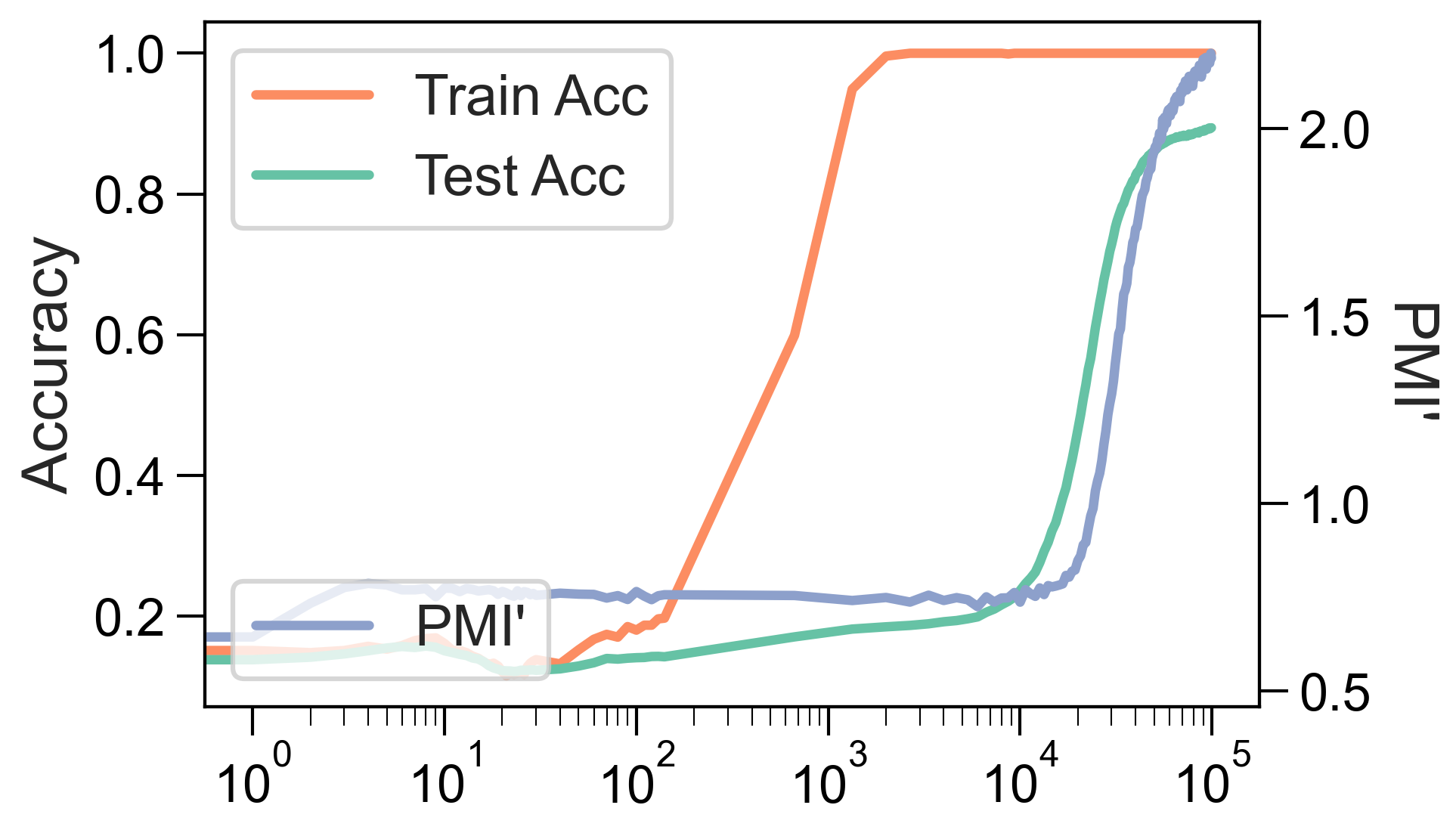}
\caption{Perturb mutual information ablation.}
\label{fig: MNIST MI ablation}
\end{figure}

\end{document}